\documentclass[twoside,11pt]{article}
\usepackage{jmlr2e}
\usepackage[export]{adjustbox}
\usepackage[latin9]{inputenc}
\usepackage[T1]{fontenc}
\usepackage{microtype}
\usepackage{mathptmx}
\usepackage{stmaryrd}
\usepackage{stackrel}
\usepackage{booktabs}
\usepackage{makecell}
\usepackage{setspace}
\usepackage{amsmath}
\usepackage{accents}
\usepackage{dcolumn}
\usepackage{diagbox}
\usepackage{subfig}
\usepackage{float}
\usepackage[english]{babel}
\usepackage{tikz}

\newcolumntype{.}{D{.}{.}{-1}}
\newcommand\thickbar[1]{\accentset{\rule{.4em}{.8pt}}{#1}}

\newcommand{\esssup}{\mathop{\mathrm{ess}\mbox{-}\mathrm{sup}}}

\ShortHeadings{Cross-validation stability}{Xu et al.}
\firstpageno{1}

\begin{document}
\title
{ $\left( \beta, \varpi \right)$-stability for cross-validation and \\ the choice of the number of folds}

\author{\name Ning Xu \email n.xu@sydney.edu.au \\
		\addr School of Economics\\
       	University of Sydney
       \AND
		\name Jian Hong \email jian.hong@sydney.edu.au \\
        \addr School of Economics\\
       	University of Sydney
       \AND
       	\name Timothy C.G. Fisher \email tim.fisher@sydney.edu.au \\
       	\addr School of Economics\\
       	University of Sydney}

\editor{}

\maketitle


\begin{abstract}
In this paper, we introduce a new concept of stability for cross-validation, called the $\left( \beta, \varpi \right)$-stability, and use it as a new perspective to build the general theory for cross-validation. The $\left( \beta, \varpi \right)$-stability mathematically connects the generalization ability and the stability of the cross-validated model via the Rademacher complexity. Our result reveals mathematically the effect of cross-validation from two sides: on one hand, cross-validation picks the model with the best empirical generalization ability by validating all the alternatives on test sets; on the other hand, cross-validation may compromise the stability of the model selection by causing subsampling error. Moreover, the difference between training and test errors in q\textsuperscript{th} round, sometimes referred to as the generalization error, might be autocorrelated on q. Guided by the ideas above, the $\left( \beta, \varpi \right)$-stability help us derivd a new class of Rademacher bounds, referred to as the one-round/convoluted Rademacher bounds, for the stability of cross-validation in both the i.i.d.\ and non-i.i.d.\ cases. For both light-tail and heavy-tail losses, the new bounds quantify the stability of the one-round/average test error of the cross-validated model in terms of its one-round/average training error, the sample sizes $n$, number of folds $K$, the tail property of the loss (encoded as Orlicz-$\Psi_\nu$ norms) and the Rademacher complexity of the model class $\Lambda$. The new class of bounds not only quantitatively reveals the stability of the generalziation ability of the cross-validated model, it also shows empirically the optimal choice for number of folds $K$, at which the upper bound of the one-round/average test error is lowest, or, to put it in another way, where the test error is most stable. 
\end{abstract}


\section{Introduction}

In many applications of statistical learning theory, overfitting arises as a common problem for learning algorithms. To reduce overfitting and improve the generalization ability of the learning algorithm, cross-validation \citep{stone74,stone77} is widely applied to regularize the in-sample performance of the algorithm. Decades of research has produced several variants of cross-validation, such as $K$-fold, repeated $K$-fold \citep{kim2009estimating}, stratified \citep{kohavi1995study}, and Monte Carlo cross-validation \citep{xu2001monte}. Simulation studies have shown that the performance of various cross-validation schemes differs according to the setting, leading to different results about the appropriate form of cross-validation.\footnote{See, for example, \citep{kohavi1995study,bengio2004no}.} At the heart of the matter lies the stability of cross-validation: when applied to sparse modeling methods (such as Lasso), cross-validation often leads to models that are unstable in high-dimensions, and consequently ill-suited for reliable interpretation \citep{lim2016estimation}. In addition, while it is a widely-applied tool for model selection and model evaluation, the general theoretical properties of cross-validation have yet to be fully worked out.

Generally speaking, as suggested by \citet{bousquet2002stability}, the more stable the predicted model on different samples, the more `generalizable' the in-sample learning result is to other samples or to the population. Hence, model complexity, algorithmic stability and generalization ability are deeply connected to one another. Guided by this instinct, our aim is to study the general theoretical properties of cross-validation from the perspective of stability, defined as the (probabilistic) maximal difference between the in-sample error (training error) and the out-of-sample error (the risk on a random new-coming sample), given the model class $\Lambda$. In cross-validation, the random new-coming sample(s) are referred to as the test sets in all $K$ rounds. To formalize it mathematically, we define the \emph{one-round} and \emph{average $\left( \beta, \varpi \right)$-stability} in definition~\ref{defn:stability}:
%
%
\begin{definition}[One-round and average $\left( \beta, \varpi \right)$-stability in K-fold cross validation]
  Let's denote $\left( Y_t^q, X_t^q \right)$ and $\left( Y_s^q, X_s^q \right)$ as the training and test sets, repesctively, in the q\textsuperscript{th} round of cross-validation. Let $n_t$ and $n_s$ be the sizes of the training and test sets respectively. $\mathcal{R}_{n_t} \left( b, Y_t^q, X_t^q \right)$ and $ \mathcal{R}_{n_s} \left( b, Y_s^q, X_s^q \right)$ are denoted as the training error and test error, respectively, for the $q$\textsuperscript{th} round of cross-validation. Given the model class $\Lambda$, $\forall \left( \beta , \varpi \right) \in \mathbb{R}^+ \times \left( 0, 1 \right]$, the functions in $\Lambda$ is \emph{one-round $\left( \beta, \varpi \right)$-stable in $K$-fold cross-validation} if
  \begin{equation}
    \mathrm{Pr} \left\{ \sup_{ b \in \Lambda } \; \left\vert \mathcal{R}_{n_t} \left( b, Y_t^q, X_t^q \right) - \mathcal{R}_{n_s} \left( b, Y_s^q, X_s^q \right) \right\vert \geqslant \beta \right\} 
    \leqslant \varpi.
    \label{eq:one_round_stability}
  \end{equation}
  \noindent
  Likewise, the functions in $\Lambda$ is \emph{average $\left( \beta, \varpi \right)$-stable in $K$-fold cross-validation} if
  \begin{equation}
    \mathrm{Pr} \left\{ \sup_{ b \in \Lambda} \; \left\vert \frac{1}{K} \sum_{q = 1}^{K} \mathcal{R}_{n_t} \left( b, Y_t^q, X_t^q \right)- \frac{1}{K} \sum_{q = 1}^{K} \mathcal{R}_{n_s} \left( b, Y_s^q, X_s^q \right) \right\vert \geqslant \beta \right\} 
    \leqslant \varpi.
    \label{eq:average_stability}
  \end{equation}
\label{defn:stability}
\end{definition}
\noindent
As we show in the proof, given the value of $\varpi$, the value of $\beta$ for either one-round or average stability may be approximated by the Rademacher complexity, $n_t$, $n_s$ and the tail heaviness of the risk distribution. Inspired by the algorithmic stability in \citet{bousquet2002stability}, the one-round or average $\left( \beta, \varpi \right)$ stability help us to quantify the connection between the stability and the generalization ability of cross-validation via model complexity. For both light-tail and heavy-tail losses, we mathematically quantify the trade-off in cross-validation via Rademacher complexity. On one hand, cross-validation improves the generalization ability of the algorithm via validating the learning result on the new-coming samples; on the other hand, it also compromises the generalization ability by making the learning result more instable: extra amount of subsampling error and extra autocorrelation between each round of cross-validation, both considered as the noise as to the `true' pattern, are injected into the learning result. As a result, given the total sample size $n$ and functional space $\Lambda$, to quantify the positive and negative effects of cross-validation on the learning result, we build a new class of bounds for the stability of cross validation. The new bounds in turn help us to define the optimal choice for the number of folds in cross-validation from perspective of stability, at which the upper bound for the average or one-round test error is lowest (or the average or one-round test error is most stable). 

\subsection{Classical generalization bounds}

The classical approach to deriving suitable generalization error bounds yielded the probably approximately correct bounds (PAC bounds), proposed in \citet{valiant1984theory} and developed by Vapnik, Chevonenkis and others. Derived from the classical concentration inequalities, the distribution- and model-free PAC-bounds characterize the relation between population risk and empirical in-sample risk from the perspective of the complexity of a given model class. The PAC bounds are useful for both model evaluation and model selection. Typically, the bounds take the form of a sum of two terms: a sample-based estimate of the in-sample risk and a term penalizing model complexity. One measure of complexity, Rademacher complexity \citep{bartlett2002rademacher}, is defined as follows.
%
%
\begin{definition}[Rademacher complexity]
  Let $P$ be a probability distribution on a set $\mathcal{X}$ and suppose that $X_1, \ldots, X_l$ are independent samples selected on $P$. Let $\Lambda$ be the model class in which all models map from $\mathcal{X}$ to $\mathbf{R}$. Define the random variables
  \begin{equation}
    \widehat{ \mathrm{ RC } } \left( \Lambda \right)
      =
      \mathbf{E}\left[ \sup_{g \in \Lambda} \left\vert \frac{2}{l} \sum_{i=1}^{l} \sigma_i g\left( X_i\right)\right\vert \; \vert \; X_1, \ldots, X_l\right]
  \end{equation}
  where $\sigma_1, \ldots, \sigma_l$ are independent uniform $\left\{ \pm 1 \right\}$-valued random variables. The Rademacher complexity of $\Lambda$ is $\mathrm{RC}_{l}\left( \Lambda \right) = \mathbf{E} \left[ \widehat{\mathrm{RC}}_{l}\left( \Lambda \right) \right]$.
\label{defn:rad_complexity}
\end{definition}
\noindent
The intuition for $\mathrm{RC}_{l}$ is straightforward. It quantifies the extent to which some function in $\Lambda$ is correlated with a binomial noise process of length $l$.\footnote{Another measure of complexity, Gaussian complexity, is closely related to Rademacher complexity. \citet{bartlett2002rademacher} show that Gaussian complexity is bounded by $\mathrm{RC}_{l}$. Hence, we focus on Rademacher complexity.} Intuitively, in finite samples, a complicated model will fit the noise no worse than a less complicated model. The advantages of the Rademacher complexity measure $\mathrm{RC}_{l}$ are that: (1) the bounds are typically `sharper' than the classical VC-bounds; (2) the bounds may be generalized to non-i.i.d.\ data such as stationery $\beta$-mixing processes \citep{shalizi2011,mohri2009rademacher,mohri2010stability}. By contrast, the VC dimension measure requires an i.i.d.\ assumption on the predictors. Hence, the performance of any model in the given model class can be measured and regularized by VC-bounds and Rademacher-bounds.

\subsection{Summary of major ideas}

By generalizing a pattern learned from a training sample to another finite sample, cross-validation obtains the training error and test error for the model class. To put it in another way, through out-of-sample validation, cross-validation captures the empirical stability of the model by finding the difference between the training and test errors. On one hand, if both training and test samples are random and identically distributed, we would expect the information encoded in both samples to be similar, implying that training and test errors should be not very different; on the other hand, with different values of $n$, $K$ and different random splitting schemes, the difference between the training and test errors may vary due to the sampling and sub-sampling error. Moreover, since $T_q$, the absolute difference between the training and test errors in q\textsuperscript{th} round of validation, may be autocorrelated as a result of the sample splitting scheme, it may cause the problem that the empirical generalization ability, more importantly the order of generalization abilities among model classes, might not be very generalizable to other new-coming data sets. Thus, given the model class $\Lambda$, $n$ and tail property of the loss, to study the `real' stability of cross-validation with the presence of the autocorrelation and sampling/subsampling error, one method is to numerically approximate its $\left( \beta, \varpi \right)$-stability, for which the empirical average test error is not sufficient.

Guided by these ideas, we analyze the stability of a cross-validated model and build the connection between stability and generalization ability of the cross-validated model via the Rademacher complexity. The major probability tool we use is `blocking'. For the i.i.d.\ case, we approximate the performance of the cross-validated model via its performance on all blocks (folds); for the non-i.i.d.\ case, we use independent blocks to reduce the disruption of the cross-correlation in data generating process. To make sure the connection between stability and generalization ability is valid via Rademacher complexity for different kinds of losses, we encode the tail property of the loss into the Orlicz-$\Psi_\nu$ norm ($\nu \in [1, 2]$) and show that the result in our paper covers both sub-exponential and sub-gaussian cases.\footnote{ the Orlicz-$\Psi_\nu$ norm is actually a quasi-norm when $\nu \in (0,1)$. However, different concertration inequalities may also be derived with $\nu \in (0,1)$. See \citet{adamczak2008tail, lederer2014new} for example.} With $\left( \beta, \varpi \right)$-stability, blocking, convolution and Orlicz-$\Psi_\nu$ norms, we derive a class of bounds, referred to as one-round/convoluted Rademacher bounds. The convoluted Rademacher bounds quantify the stability (specifically the upper bound) of the average test error via $K$, $n$, the Orlicz-$\Psi_\nu$ norm of the loss and $\varpi$. As a result, given other parameters fixed, the optimal choice of $( K^*, b^*) \in \left\{ [0,n] \; , \; \Lambda \right\}$ is obtained at the minimum point of the convoluted Rademacher bound, at which the upper bound of the average test error is minimal. To put it in another way, the average test error is most stable at $( K^*, b^*)$.

The paper is organized as follows. In section~2, based on the one-round/average $\left( \beta, \varpi \right)$-stability, we list our assumptions and proceed to derive the one-round and convoluted Rademacher-bounds for cross-validation under traditional i.i.d. settings. In section~3, using the independent blocks, we construct the one-round and convoluted Rademacher bounds for cross-validation under the $\beta$-mixing scenario.


\section{Convoluted Rademacher-bounds for the cross-validated test error in i.i.d\ data}

\subsection{Notation and assumptions}

Our notation is defined as follows.
\begin{definition}
[Subsamples, training error, and test error]\label{def:notation}
\end{definition}

  \begin{enumerate}

    \item   Let $(y, \mathbf{x})$ denote a sample point from $F(y, \mathbf{x})$, where $F(y,\mathbf{x})$ is the joint distribution of $(y,\mathbf{x})$. Given a sample $(Y,X)$, the \textit{training set} is denoted as $(Y_{t},\,X_{t})\in\mathbb{R}^{n_t \times p}$ and the \textit{test set} as $(Y_{s},\,X_{s})\in\mathbb{R}^{n_s \times p}$. Let $\widetilde{n}=\min\{n_s,n_t\}$.

    \item   Let $\Lambda$ denote a model class. The \textit{loss function} for a model $b\in\Lambda$ is $Q \left( b, y_i,\mathbf{x}_i \right),\,i=1,\ldots,n$. The \textit{population risk} for $b\in\Lambda$ is $\mathcal{R} \left( b, Y, X \right) = \int Q\left( b, y,\mathbf{x} \right) \mathrm{d} F \left(y ,\mathbf{x} \right)$. The \textit{empirical risk} is $\mathcal{R}_{n} \left( b, Y, X \right) = \frac{1}{n}\;\sum_{i=1}^n \; Q \left( b, y_i, \mathbf{x}_i \right)$.

    \item   The \textit{training error}, for any $b \in \Lambda$, is $\mathcal{R}_{n_t} \left( b, Y_{t}, X_{t} \right)$. The \textit{test error}, for any $b \in \Lambda$, is $\mathcal{R}_{n_s} \left( b, Y_{s}, X_{s} \right)$. The population risk for $b$ is $\mathcal{R} \left( b, Y, X \right)$.

    \item For $K$-fold cross-validation, denote the training set and test set in the q\textsuperscript{th} round, respectively, as $(X_{t}^q,Y_{t}^q)$ and $(X_s^q,Y_s^q)$. As a result, the training / test error on the $q$\textsuperscript{th} training / test set is $\mathcal{R}_{n_t} \left(X_t^q,Y_t^q \right)$ and $\mathcal{R}_{n_s} \left( X_s^q,Y_s^q \right)$, respectively

    \item We define the Orlicz-$\Psi_\nu$ norm for an empirical process $Z$ as
      \[
        \left\Vert Z \right\Vert_{\Psi_\nu} := \inf_{u > 0} \left\{ u \; \left\vert \; \mathbb{E} \left[ \exp \left\{ \frac{\left\vert Z \right\vert^\nu } { u^\nu } \right\} \right] < 2 \right\}. \right.
      \]
      We also define the empirical process, $\forall b \in \Lambda, \; \forall q \in \left[1,K\right]$,
      \begin{align}
        U_q & := \sup_{b \in \Lambda} \; \left\vert
          \mathcal{R}_{n_s} \left( b, Y_{s}^q, X_{s}^q \right) - \mathcal{R}_{n_t} \left( b, Y_{t}^q, X_{t}^q \right)
          \right\vert \notag \\
        T_q & := U_q - \mathbb{E} \left[ U_q \right]. \notag
      \end{align}
  \end{enumerate}

Our main assumptions are listed as follows.

\bigskip
\noindent
\textbf{Assumptions}

\begin{enumerate}
    \item[\textbf{A1.}]   In the probability space $\left( \Omega, \mathcal{F}, P \right)$ for the loss function, we assume $\mathcal{F}$-measurability of the loss $Q(b, y, \mathbf{x})$, the population error $\mathcal{R}(b,Y,X)$, and the empirical error $\mathcal{R}_{n}(b, Y,X)$, for any $b\in\Lambda$ and any sample point $(y, \mathbf{x})$. Specifically, we assume that the Orlicz-$\Psi_\nu$ norm of all the loss processes are always well-defined for $\nu \geqslant 1$.
    \item[\textbf{A2.}]   All points in each fold for cross-validation are randomly partitioned. The data points of $(Y,X)$ are independently sampled from the population.
    \item[\textbf{A3.}]   \citep{vapnik1998statistical} We assume that, in the model class $\Lambda$, there exists a countable number of sub-classes,\footnote{The subclasses are equivalence classes w.r.t. the $\mathrm{RC}_{n/K}$.} denoted as $ \left\{ \mathcal{C}_i \; \vert \; i = 1, \ldots, I \right\}$, such that $ \cup_{i} \mathcal{C}_i = \Lambda$. We assume the existence of an admissible structure: a linear order exists within the set $ \left\{ \mathrm{RC}_{n/K} \left( \mathcal{C}_i \right) \; \vert \; i = 1, \ldots, I \right\}$.
    \item[\textbf{A4.}]   For all the empirical processes in this paper, the VC entropy in model class $\Lambda$, $H^{\Lambda} \left( \epsilon, n \right)$, satisfies $\lim_{n \rightarrow \infty} H^{\Lambda} \left( \epsilon, n \right) / n = 0,\; \forall \epsilon > 0$
    \item[\textbf{A5.}]   (\citep{jayan1985commutability} interchangability of $\sup \left[ \cdot \right]$ and $\mathbb{E} \left[ \cdot \right]$) ~ Let's denote $W_n \left( b \right) = \left( \mathcal{R}_n \left(b, X, Y \right) - \mathcal{R} \left(b, X, Y \right) \right)^2 $. Assume $\sup_{b \in \Lambda } \mathbb{E} \left[ W_n \left( b \right) \right] \leqslant \infty$. $ \forall \epsilon \geqslant 0$ and $\alpha, \delta \in \Lambda $, $\exists \upsilon \in \Lambda$ such that 
    \[
      \mathbb{E} \left[ \left( W_n \left( \upsilon \right) - W_n \left( \alpha \right) \vee W_n \left( \delta \right) \right)^- \right] \leqslant \epsilon.
    \]
\end{enumerate}

Several remarks apply to the assumptions. Assumption~A1 on the loss distribution formalizes the analysis of test errors. The Orlicz-$\Psi_\nu$ norm encodes the tail property of the loss distribution. If $\nu = 2$, the distribution is in the subgaussian family and large losses do not occur very often; if $\nu = 1$, the distribution is in the subexponential family and the tails are `heavier' than for the subgaussians. A2~needs to be clarified specifically. In this section, we focus on the independence case; we relax the independence restriction in the next section to allow for mixing processes. The admissible structure in assumption~A3 originates in the work of Vapnik in 1960s and rules out the scenario in which the complexity of models cannot be compared pairwise. The VC entropy assumption~A4 ensures the empirical risk of any model in the model class is consistent (see \citet{vapnik1998statistical}). Lastly, assumptions~A5 are known to be the most general condition for interchangability of $\sup \left[ \cdot \right]$ and $\mathbb{E} \left[ \cdot \right]$, which are well-known and frequently used in stochastic optimization, backwards stochastic differential equations and other financial math researches.\footnote{ This condition is proved by \citet{striebel2013optimal} and generalized by \citet{jayan1985commutability}. See \citet[p250]{elliott1982stochastic} for detail. Generally speaking, in mathematics $\mathbb{E} \left[ \sup_{b \in \Lambda} \left[ \cdot \right]\right]$ may not be measurable. As a result, \citet{jayan1985commutability} use the operator $\esssup \left[ \cdot \right]$ instead. However, if it is measurable, the $\esssup$ and $\sup$ are identical. For the application of this condition, see \citet{peng2004filtration} as an example.} Another method to partially solve the interchangability is inspired by \citet{vapnik1998statistical})\footnote{ this method is originally used by Vapnik (for the proof of the VC bounds for the heavy-tail loss) to control the ratio between $L_1$ norm and $L_p$ norm of the random variables with heavy tails, where $ 2 > p > 1$.}:
  \begin{enumerate}
    \item[\textbf{A5$'$.}]  Given the model class $\Lambda$, 
      \begin{align}
        \sup_{b \in \Lambda} \left\{ 
          \frac{ \mathbb{E} \left[ \sup \left( \mathcal{R}_n \left( b, X, Y \right) - \mathcal{R} \left( b, X, Y \right) \right)^2 \right] }
          { \sup \left\{ \mathbb{E} \left[ \left( \mathcal{R}_n \left( b, X, Y \right) - \mathcal{R} \left( b, X, Y \right) \right)^2 \right] \right\} } \right\}
          \leqslant \zeta \in \mathbb{R}^+.
      \end{align}
  \end{enumerate}
\noindent
Either A5 or A5$'$ could help us interchange the order of $\sup \left[ \cdot \right]$ and $\mathbb{E} \left[ \cdot \right]$. In this paper, for the mathematical elegance, we derive the theoretical results under A5; however,  a set of very similar results could also be derived under A5$'$, which is skipped here.

\subsection{Convoluted Rademacher-bounds for the test errors}

Firstly, we show in Theorem~\ref{thm:one_round_VC_bound} that, in each round of cross-validation, the one-round test error is bounded by a Rademacher-bound, which depends on $n$, $\mathrm{RC}$ and $\varpi$. To do that, we use the concentration inequality in Orlicz-$\Psi_\nu$ norms. The reasons we use that specific inequality is because: (1) it is an exponential concentration inequality, which means we can build a tighter bound than the non-exponential inequalities; (2) the Orlicz-$\Psi_\nu$ norm reveals the tail-behavior or concentration of the subexponential variables (and the heavy/flat/long-tail variables) in a more straightforward way than the Lebesgue-$p$ norm. It also applies to case where the $L_2$ norm of the random variable does not exist.
%
%
\begin{theorem}[Rademacher-bound of the one-round test error in cross-validation]
  In the q\textsuperscript{th} round of validation, under \textbf{A1} to \textbf{A5}, the following upper bound for the test error holds with probability at least $ 1 - \varpi \in \left( 0 , 1 \right]$, $\forall b \in \Lambda$.
  \noindent
  \begin{equation}
    \mathcal{R}_{n_s} \left( b, Y_s^q, X_s^q \right) \leqslant \mathcal{R}_{n_t}
      \left( b, Y_t^q, X_t^q \right) + 2 \cdot \mathrm{RC}_{n/K} \left( \Lambda \right) + \varsigma,
    \label{eq:VC_bound}
  \end{equation}
  \noindent
  where $\rho_j = \sup_{b \in \Lambda} \left\vert  \sum_{i = 1}^{ n/K } Q \left( b, y^{j,i}, \mathrm{x}^{j,i} \right) / ( n/K ) - \mathcal{R} \left( b, Y, X \right) \right\vert$, \, $\sigma^2$ is the variance proxy of $\rho_j$ if $\rho_j$ is subgaussian, \, $ \xi^2 : = \sigma^2 / \mathrm{var} \left[ \rho_j \right] $, $c$ is an absolute constant in the exponential inequality by \citet{Lecue09tool} \footnote{ The absolute constant $c$ is from Theorem 0.4 in \citet{Lecue09tool}, which is slightly modified from \citet{talagrand1994supremum}. Different papers have investigated the value of the absolute constant in the Talagrand's concerntration inequality, for example see \citet{taoconstant, massart2000constants}.}, and
  \[
    \varsigma =
      \left\{
        \begin{array}{ll}
          2 \cdot M \cdot \sqrt{ \frac{ \log \left( 1 / \varpi \right) } { n / K } },
          & \mbox{ if } \sup_{ b \in \Lambda } \left( Q \right) \leqslant M \\
          2 \cdot \xi \cdot \tilde{\sigma} \cdot \sqrt{ \frac{ \log \left( 1 / \varpi \right) } { n / K } }
          & \mbox{ if } \rho_j \mbox{ is subgaussian } \\
          \left\Vert \rho_j \right\Vert_{\Psi_1} \cdot \log \sqrt[c] {\left( 2 / \varpi \right) } ,
          & \mbox{ if } \rho_j \mbox{ is subexponential and } ~ 2\exp \left\{ -2 c \right\} > \varpi  >  0  \\
          \left\Vert \rho_j \right\Vert_{\Psi_1} \cdot \left( 2 \cdot \log \sqrt[c] { \left( 2 / \varpi \right) } \right)^{\frac{1}{2}} ,
          & \mbox{ if } \rho_j \mbox{ is subexponential and } ~ 2\exp \left\{ -2 c \right\} \leqslant \varpi \leqslant 1
        \end{array}\right.
  \]
\label{thm:one_round_VC_bound}
\end{theorem}
%
%
\begin{proof}

To prove Theorem~\ref{thm:one_round_VC_bound}, we start from quantifying the following probability,
\noindent
\begin{equation}
  \mathrm{Pr} \left[ \sup_{b \in \Lambda} \left\vert \mathcal{R}_{n_s} \left( b, Y_s^q, X_s^q \right) - \mathcal{R}_{n_t} \left( b, Y_t^q, X_t^q \right) \right\vert > \epsilon \right],
  \mbox{ for given } \epsilon \in \mathbb{R}^+.
\end{equation}
\noindent
This probability may be simplified as follows,

\noindent
\begin{align}
  \phantom{=} & ~ \mathrm{Pr} \left[ \sup_{b \in \Lambda} \left\vert \mathcal{R}_{n_s} \left( b, Y_s^q, X_s^q \right) - \mathcal{R}_{n_t} \left( b, Y_t^q, X_t^q \right) \right\vert > \epsilon \right] \notag \\
  = & ~ \mathrm{Pr} \left[ \sup_{b \in \Lambda} \left\vert \mathcal{R}_{n_s} \left( b, Y_s^q, X_s^q \right) - \mathcal{R} \left( b, Y, X \right) + \mathcal{R} \left( b, Y, X \right) - \mathcal{R}_{n_t} \left( b, Y_t^q, X_t^q \right) \right\vert > \epsilon \right].
\label{eq:proof_1_1}
\end{align}

\noindent
Due to the convexity of the norm, eq.~(\ref{eq:proof_1_1}) implies that

\noindent
\begin{align}
  \phantom{=} & ~ \mathrm{Pr} \left[ \sup_{b \in \Lambda} \left\vert \mathcal{R}_{n_s} \left( b, Y_s^q, X_s^q \right) - \mathcal{R} \left( b, Y, X \right) + \mathcal{R} \left( b, Y, X \right) - \mathcal{R}_{n_t} \left( b, Y_t^q, X_t^q \right) \right\vert > \epsilon \right] \notag \\
  \leqslant & ~ \mathrm{Pr} \left[ \sup_{b \in \Lambda} \left\vert \mathcal{R}_{n_s} \left( b, Y_s^q, X_s^q \right) - \mathcal{R} \left( b, Y, X \right) \right\vert + \sup_{b \in \Lambda} \left\vert \mathcal{R} \left( b, Y, X \right) - \mathcal{R}_{n_t} \left( b, Y_t^q, X_t^q \right) \right\vert > \epsilon \right].
\label{eq:proof_1_2}
\end{align}

\noindent
If we further define
\begin{align}
  \Phi_{n_t} := & \sup_{b \in \Lambda} \left\vert \mathcal{R} \left( b, Y, X \right) - \mathcal{R}_{n_t} \left( b, Y_t^q, X_t^q \right) \right\vert, \\
  \Phi_{n_s} := & \sup_{b \in \Lambda} \left\vert \mathcal{R}_{n_s} \left( b, Y_s^q, X_s^q \right) - \mathcal{R} \left( b, Y, X \right) \right\vert,
\end{align}

\noindent
the union bound implies that for eq.~(\ref{eq:proof_1_2}) the following derivation holds

\noindent
\begin{align}
  \phantom{=} & ~ \mathrm{Pr} \left[ \sup_{b \in \Lambda} \left\vert \mathcal{R}_{n_s} \left( b, Y_s^q, X_s^q \right) - \mathcal{R} \left( b, Y, X \right) \right\vert + \sup_{b \in \Lambda} \left\vert \mathcal{R} \left( b, Y, X \right) - \mathcal{R}_{n_t} \left( b, Y_t^q, X_t^q \right)\right\vert > \epsilon \right] \notag \\
  = & ~ \mathrm{Pr} \left[ \Phi_{n_s} + \Phi_{n_t} > \epsilon \right]
\end{align}

\smallskip
\noindent
1. \emph{$\sup_{b\in\Lambda}Q\left(b\right)$ is bounded by $M$}

\smallskip
If we define the performance of the model class in each fold as
\[
  \rho_j = \sup_{b \in \Lambda} \left\vert \frac{ 1 } { n/K } \sum_{i = 1}^{ n/K } Q \left( b, y^{j,i}, \mathrm{x}^{j,i} \right) - \mathcal{R} \left( b, Y, X \right) \right\vert,
\]

\noindent
the following relation between $\Phi_{n_t}$ and $\rho_j$ holds,
\begin{align}
  \Phi_{n_t} & \leqslant  \frac{ 1 } { K - 1 } \sum_{j = 1}^{ K - 1 } \sup_{b \in \Lambda} \left\vert \frac{ 1 } { n/K } \sum_{i = 1}^{ n/K } Q \left( b, y^{j,i}, \mathrm{x}^{j,i} \right) - \mathcal{R} \left( b, Y, X \right) \right\vert   \\
  & = \frac{ 1 } { K - 1 } \sum_{q = 1}^{ K - 1 } \rho_j \, .
\end{align}

\noindent
If $ \sup_{ b \in \Lambda } Q \left( b \right) \leqslant M $, $\rho_j$ is also bounded by $M$, implying that $\rho_j$ is a subgaussian with variance proxy $\sigma$. Define $\thickbar{\epsilon} = \epsilon - 2 \cdot \mathbb{E} \left( \rho_j \right)$ and the fold used as the test data as $\rho_K$. We also denote the vector $\mathbf{w} = \left[ 1 / \left( K - 1 \right), \ldots, 1 / \left( K - 1 \right), 1 \right]$ as the weight and $\left\Vert \cdot \right\Vert_2$ as the $L^2$ norm. As a result, the Chernoff bounds imply that

\noindent
\begin{align}
  \mathrm{Pr} \left[ \Phi_{n_t} + \Phi_{n_s} > \epsilon \right] & \leqslant \mathrm{Pr} \left[ \frac{ 1 } { K - 1 } \sum_{ j = 1 }^{ K - 1 } \rho_j + \rho_K > \epsilon \right] \notag \\
  & = \mathrm{Pr} \left[ \frac{ 1 } { K - 1 } \sum_{j = 1}^{ K - 1 } \rho_j + \rho_K - 2 \cdot \mathbb{E} \left( \rho_j \right) > \epsilon - 2 \cdot \mathbb{E} \left( \rho_j \right) \right] \\
  & \leqslant \exp \left\{ - \frac{ \thickbar{\epsilon}^2 } { 2 \cdot \left\Vert \mathbf{w} \right\Vert_2^2 \cdot \sigma^2 } \right\} \\
  & =  \exp \left\{ - \frac{ \thickbar{\epsilon}^2 \cdot \left( K - 1 \right)} { 2 \cdot \sigma^2 \cdot K} \right\}
\end{align}

\noindent
Since $K \in \left[ 2, n \right]$,

\noindent
\begin{align}
  \mathrm{Pr} \left[ \sup_{b \in \Lambda} \left\{ \mathcal{R}_{n_s} \left( b, Y_s^q, X_s^q \right) - \mathcal{R}_{n_t} \left( b, Y_t^q, X_t^q \right) \right\} > \epsilon \right]
  & \leqslant \exp \left\{ -\frac{ \left( \epsilon - 2 \cdot \mathbb{E} \left[ \rho_j \right] \right)^2 \cdot \left( K - 1 \right) }{ 2 \cdot \sigma^2 \cdot K } \right\} \\
  & \leqslant \exp \left\{ -\frac{ \left( \epsilon - 2 \cdot \mathbb{E} \left[ \rho_j \right] \right)^2 }{ 4 \cdot \sigma^2 } \right\}
\end{align}

\noindent
If we set $ \varpi = \exp \left\{ - \left( \epsilon - 2 \cdot \mathbb{E} \left[ \rho_j \right] \right)^2 / \left( 4 \cdot \sigma^2 \right) \right\}$,

\noindent
\begin{align}
  \epsilon & = 2 \cdot \mathbb{E} \left[ \rho_j \right] + 2 \cdot \sigma \cdot \sqrt{ \log \left( \frac{ 1 } { \varpi } \right) } \\
  \phantom{ \epsilon } & \leqslant 2 \cdot \mathrm{RC}_{n/K} \left( \Lambda \right) + 2 \cdot \sigma \cdot \sqrt{ \log \left( \frac{ 1 } { \varpi } \right) }
\end{align}
\noindent
where $\mathrm{RC}_{n/K} \left( \Lambda \right) $ is the Rademacher complexity on each fold. Hence, $ \forall b \in \Lambda $, the following inequality holds with probability at least $ 1 - \varpi \in \left( 0 , 1  \right] $,

\noindent
\begin{equation}
  \mathcal{R}_{n_s} \left( b, Y_s^q, X_s^q \right) \leqslant \mathcal{R}_{n_t} \left( b, Y_t^q, X_t^q \right) + 2 \cdot \mathrm{RC}_{n/K} \left( \Lambda \right) + 2 \cdot \sigma \cdot \sqrt{ \log \left( \frac{ 1 } { \varpi } \right) }.
\end{equation}

\noindent
Denote $ \sigma^2 / \mathrm{var} \left[ \rho_j \right] = \xi^2 $. Based on assumption A5,\footnote{ If we use assumption A5$'$, an extra $\zeta \in \mathbb{R}^+$ will occur besides $M$. We rely on A5 in this paper, but similar results could be derived under A5$'$, which is straightforward and will not be specified here.} the supremum operator and expectation operator may be interchanged. Hence, the definition of $\rho_i$ implies that
\noindent
\begin{align}
  \sigma^2 \leqslant & ~ \xi^2 \cdot \mathbb{E} \left[ \sup_{b \in \Lambda} \left\vert \frac{1}{n/K} \sum_{i=1}^{n/K} \left[ Q \left( b, y^{j,i}, \mathbf{x}^{j,i} \right) - \mathcal{R} \left( b, Y, X \right) \right] \right\vert \right]^2 \\
  = & ~ \xi^2 \cdot \mathbb{E} \left[ \sup_{b \in \Lambda} \left( \frac{1}{n/K} \sum_{i=1}^{n/K} \left[ Q \left( b, y^{j,i}, \mathbf{x}^{j,i} \right) - \mathcal{R} \left( b, Y, X \right) \right] \right)^2 \right] \\
  = & ~ \xi^2 \cdot \sup_{b \in \Lambda} \left( \mathbb{E} \left[ \frac{1}{n/K} \sum_{i=1}^{n/K} \left[ Q \left( b, y^{j,i}, \mathbf{x}^{j,i} \right) - \mathcal{R} \left( b, Y, X \right) \right] \right]^2 \right)  \\
  = & ~ \xi^2 \cdot \frac{1}{n/K} \cdot \sup_{b \in \Lambda} \left\{ \mathrm{var} \left[ Q \left( b, y^{j,i}, \mathbf{x}^{j,i} \right) - \mathcal{R} \left( b, Y, X \right) \right] \right\}
  \label{eqn:last_step_bounded} \\
  \leqslant & ~  \frac{ M^2 }{n/K} .
\end{align}
\noindent
As a result, given the probability $1-\varpi$,

\noindent
\begin{equation}
  \mathcal{R}_{n_s} \left( b, Y_s^q, X_s^q \right) \leqslant \mathcal{R}_{n_t} \left( b, Y_t^q, X_t^q \right) + 2 \cdot \mathrm{RC}_{n/K} \left( \Lambda \right) + 2 \cdot M \cdot \sqrt{ \frac{ \log \left( 1 / \varpi \right) } { n / K } }.
\end{equation}

\noindent
The RHS and LHS will be identical if $\lim_{n \rightarrow \infty} \, K / n = 0$.

\bigskip

\noindent
2. \emph{$ \left\Vert \rho_j \right\Vert_{\Psi_2} $ is finite}

\smallskip
If the Orlicz-$\Psi_2$ norm of $\rho_j$ is finite, $\rho_j$ is subgaussian. By defining the variance proxy of $ \sup_{ b \in \Lambda } \left\{ \mathrm{var} \left[ Q \left( b, y^{j,i}, \mathbf{x}^{j,i} \right) - \mathcal{R} \left( b, Y, X \right) \right] \right\}$ as $\tilde{\sigma}^2$, eq.~(\ref{eqn:last_step_bounded}) implies that
\noindent
\begin{align}
  \mathcal{R}_{n_s} \left( b, Y_s^q, X_s^q \right) \leqslant \mathcal{R}_{n_t} \left( b, Y_t^q, X_t^q \right) + 2 \cdot \mathrm{RC}_{n/K} \left( \Lambda \right) + 2 \cdot \xi \cdot \tilde{\sigma} \cdot \sqrt{ \frac{ \log \left( 1 / \varpi \right) } { n / K } },  \notag \\
  \forall 1-\varpi \in \left[ 0,1 \right)
\end{align}
\bigskip


\noindent
3. \emph{$\left\Vert \rho_j \right\Vert_{\Psi_1}$ is finite}

\smallskip
If $\left\Vert \rho_j \right\Vert_{\Psi_1}$ is finite, $\rho_j$ is subexopential. As a result, a Bernstein-type inequality holds as follows
\noindent
\begin{align}
  \mathrm{Pr} \left[ \Phi_{n_t} + \Phi_{n_s} > \epsilon \right] & \leqslant \mathrm{Pr} \left[ \frac{ 1 } { K - 1 } \sum_{ j = 1 }^{ K - 1 } \rho_j + \rho_K > \epsilon \right] \\
  & = \mathrm{Pr} \left[ \frac{ 1 } { K - 1 } \sum_{j = 1}^{ K - 1 } \rho_j + \rho_K - 2 \cdot \mathbb{E} \left( \rho_j \right) > \epsilon - 2 \cdot \mathbb{E} \left( \rho_j \right) \right] \\
  & \leqslant 2 \cdot \exp \left\{ - c \cdot \min \left(
    \frac{ \thickbar{\epsilon}^2 } { \left\Vert \rho_j \right\Vert_{ \Psi_1}^2 \cdot \left\Vert \mathbf{w} \right\Vert_2^2 },
    \frac{ \thickbar{\epsilon}} { \left\Vert \rho_j \right\Vert_{\Psi_1} \cdot \left\Vert \mathbf{w} \right\Vert_\infty }
    \right) \right\}  \\
  & \leqslant 2 \cdot \exp \left\{ - c \cdot \min \left(
    \frac{ \left( \epsilon - 2 \cdot \mathbb{E} \left[ \rho_j \right] \right)^2 } { 2 \cdot \left\Vert \rho_j \right\Vert_{ \Psi_1}^2 },
    \frac{ \epsilon -2 \cdot \mathbb{E} \left[ \rho_j \right] } { \left\Vert \rho_j \right\Vert_{\Psi_1} }
    \right) \right\}
\end{align}
\noindent
If we set $ \left( \epsilon -2 \cdot \mathbb{E} \left[ \rho_j \right] \right) / \left( 2 \cdot \left\Vert \rho_j \right\Vert_{ \Psi_1} \right) = \tau $,
\noindent
\begin{align}
  \min \left(
  \frac{ \left( \epsilon - 2 \cdot \mathbb{E} \left[ \rho_j \right] \right)^2 } { 2 \cdot \left\Vert \rho_j \right\Vert_{ \Psi_1}^2 },
  \frac{ \epsilon -2 \cdot \mathbb{E} \left[ \rho_j \right] } { \left\Vert \rho_j \right\Vert_{\Psi_1} }
  \right)
  = \left\{
  \begin{array}{lr}
    \left( \epsilon - 2 \cdot \mathbb{E} \left[ \rho_j \right] \right) / \left\Vert \rho_j \right\Vert_{\Psi_1} ,   & \text{if } \tau > 1  \\
    \left( \epsilon - 2 \cdot \mathbb{E} \left[ \rho_j \right] \right)^2 / \left( \sqrt{2} \cdot  \left\Vert \rho_j \right\Vert_{\Psi_1} \right)^2 ,   & \text{if } \tau \leqslant 1
  \end{array}\right.
\end{align}

\noindent
Hence,

\noindent
\begin{align}
  \mathrm{Pr} \left[ \sup_{b \in \Lambda} \left\vert \mathcal{R}_{n_s} \left( b, X_s^q, Y_s^q \right) - \mathcal{R}_{n_t} \left( b, X_s^q, Y_s^q \right) \right\vert > \epsilon \right]
  \leqslant  \left\{
    \begin{array}{lr}
      2 \exp \left\{ - c \cdot \frac{ \epsilon - 2 \cdot \mathbb{E} \left[ \rho_j \right] } { \left\Vert \rho_j \right\Vert_{\Psi_1} }\right\} ,   & \text{if } \tau > 1  \\
      2 \exp \left\{ - c \cdot \frac{ \left( \epsilon - 2 \cdot \mathbb{E} \left[ \rho_j \right] \right)^2 } { 2 \cdot \left\Vert \rho_j \right\Vert_{\Psi_1}^2 } \right\} ,   & \text{if } \tau \leqslant 1
    \end{array}\right.
  \label{eqn:proof_heavy_tail}
\end{align}

\noindent
If we set $ \varpi $ as the RHS of eq.~(\ref{eqn:proof_heavy_tail}), i.e.
\begin{align}
  \varpi = \left\{
    \begin{array}{lr}
      2 \exp \left\{ - c \cdot \frac{ \epsilon - 2 \cdot \mathbb{E} \left[ \rho_j \right] } { \left\Vert \rho_j \right\Vert_{\Psi_1} } \right\} ,   & \text{if} ~~ \epsilon - 2 \cdot \mathbb{E} \left[ \rho_j \right] >  2 \cdot \left\Vert \rho_j \right\Vert_{\Psi_1} > 0 \\
      2 \exp \left\{ - c \cdot \frac{ \left( \epsilon - 2 \cdot \mathbb{E} \left[ \rho_j \right] \right)^2 } { 2 \cdot \left\Vert \rho_j \right\Vert_{\Psi_1}^2 } \right\} ,   & \text{if} ~~ 0 < \epsilon - 2 \cdot \mathbb{E} \left[ \rho_j \right] \leqslant  2 \cdot \left\Vert \rho_j \right\Vert_{\Psi_1}
    \end{array}\right.
\end{align}
$\epsilon$ may be expressed as a function of $\varpi$,
\noindent
\begin{align}
  \epsilon= &
    \begin{cases}
      2 \mathbb{E} \left[ \rho_j \right] + \left\Vert \rho_j \right\Vert_{\Psi_1} \cdot \log \sqrt[c] {\left( 2 / \varpi \right) } , & \,\mathrm{if} ~~ \epsilon - 2 \cdot \mathbb{E} \left[ \rho_j \right] >  2 \cdot \left\Vert \rho_j \right\Vert_{\Psi_1} > 0 \\
      2 \mathbb{E} \left[ \rho_j \right] + \left\Vert \rho_j \right\Vert_{\Psi_1} \cdot \left( 2 \cdot \log \sqrt[c] { \left( 2 / \varpi \right) } \right)^{\frac{1}{2}} , & \mathrm{if} ~~ 0 < \epsilon - 2 \cdot \mathbb{E} \left[ \rho_j \right] \leqslant  2 \cdot \left\Vert \rho_j \right\Vert_{\Psi_1}
    \end{cases}
\end{align}
\noindent
Hence, $\forall b \in \Lambda$, the following inequality holds with probability at least $ 1 - \varpi \in \left( 0 , 1 \right]$
\noindent
\begin{align}
  \mathcal{R}_{n_s} \left( b, Y_s^q, X_s^q \right) \leqslant \mathcal{R}_{n_t} \left( b, Y_t^q, X_t^q \right)+ 2 \cdot \mathrm{RC}_{n/K} \left( \Lambda \right) +  \left\Vert \rho_j \right\Vert_{\Psi_1} \cdot \log \sqrt[c] {\left( 2 / \varpi \right) } , \notag \\
  \mathrm{if} ~~ 2 \exp \left\{ -2 c \right\} > \varpi  >  0 \\
  \mathcal{R}_{n_s} \left( b, Y_s^q, X_s^q \right) \leqslant \mathcal{R}_{n_t} \left( b, Y_t^q, X_t^q \right)+ 2 \cdot \mathrm{RC}_{n/K} \left( \Lambda \right) + \left\Vert \rho_j \right\Vert_{\Psi_1} \cdot \left( 2 \cdot \log \sqrt[c] { \left( 2 / \varpi \right) } \right)^{\frac{1}{2}} , \notag \\
  \mathrm{if} ~~ 2 \exp \left\{ -2 c \right\} \leqslant \varpi \leqslant 1
\end{align}
\smallskip
\end{proof}
%
%
Based on Rademacher complexity, eq.~(\ref{eq:VC_bound}) establishes the upper bound of the one-round test error based on the one-round training error for risks with different tails. The upper bound of the one-round test error is directly derived by finding the upper bound for $\varpi$ of the one-round $\left( \beta, \varpi \right)$-stability. Hence, the one-round upper bound measures the probabilistic stability of the one-round test error. Due to the Olicz norm, eq.~(\ref{eq:VC_bound}) will also work when Lebesgue-2 norm of the loss does not exist. As shown in eq.~(\ref{eq:VC_bound}), the stability of the one-round test error is determined by $n$, $K$, and the $\left\Vert \cdot \right\Vert_{\Psi_\nu}$ norm, where $\nu \in \left[ 1 , 2 \right]$. The smaller $\nu$, the heavier the tails of the loss, and the more volatile the one-round test error. Furthermore, in one-round validation, the values of $n$ and $K$ directly influence how similar are the test and training sets, or the parity between test and training sets. Leave-one-out cross-validation, for example, uses one point to validate the result learned from the other\ $n-1$ points. Given the value of $n$, leave-one-out cross-validation does not maintain the parity very well, and hence it results in unnecessary instability in one-round validation: $\mathrm{RC}_{1}$ and the last term of eq.~(\ref{eq:VC_bound}) would be extremely high, even though $\mathcal{R}_{n_t}$ may be close to its population value.\footnote{Eq.~(\ref{eq:VC_bound}) is not restricted to $K$-fold cross-validation. It may also be used to measure one-round stability of the test error in Monte Carlo cross-validation, stratified cross-validation, and repeated $K$-fold cross-validation.} Hence, it is beneficial to balance the sample sizes of the training and test sets. It is also worth nothing that since the one-round stability is defined using the supremum norm, a similar lower bound may be derived as well. Due to the similarity of the techniques, we skip the derivation of the lower bound.

The next step is to convolute the test error and training error in each round and establish the convoluted Rademacher-bounds for cross-validation. If the empirical processes $\mathcal{R}_{n_t}$ and $\mathcal{R}_{n_s}$ in all rounds of cross-validation are independent, we can directly apply concentration inequalities, such as the Hoeffding or Bernstein inequalities, and approximate the probability of the convoluted Rademacher-bounds. However, with large $K$, it is well-known that $\left\{ \mathcal{R}_{n_t} - \mathcal{R}_{n_s} \right\}$ may be an empirical process with a strong and persistent autocovariance, implying that the straightforward i.i.d. concentration paradigm cannot be applied here. In order to establish the convoluted PAC-bound for cross-validation based on one-round bounds, we first generalize the Chebyshev inequality to dependent processes in the following lemma.
%
%
\begin{lemma}
Assume $\left\{ X_i \right\}_{i=1}^{n}$ is sampled from a stationery and mean-square ergodic process and that the autocovariance function $\mathrm{cov} \left(X_{i+l},\; X_{i} \right) := \gamma_l < \infty,\, \forall l \in \mathbf{R}$. The following inequality holds for any $\varpi \in \left[0,1\right)$,

  \begin{equation}
    \mathrm{Pr} \left( \left\vert \overline{X} - \mathbb{E} \left( X \right) \right\vert \leqslant \epsilon \right)
    \geqslant 1 - \frac{ \gamma_0 }{ \epsilon^2 n} \cdot \left( 1 + 2 V_n \left[ X \right] \right).
    \label{Cheby_ineq}
  \end{equation}

\noindent
where $V_n \left[ X \right] := \sum_{l=1}^{n-1} \left\vert \gamma_l\right\vert$ and $\varpi = \gamma_0 / ( n \cdot \epsilon^2 ) \cdot \left( 1 + 2 V_n \left[ X \right] \right)$.
\label{lem:Cheby_ineq}
\end{lemma}

For the completeness, the proof of the last lemma is included in the appendix. A few comments apply to Lemma~\ref{lem:Cheby_ineq}. In eq.~(\ref{Cheby_ineq}), $\lim_{ n \rightarrow \infty} V_n \left[ X \right]$ is the autocovariance time for the process $\left\{ X_i \right\}$ in the population. As long as $\lim_{ n \rightarrow \infty} V_n \left[ X \right] / n \leqslant \infty$, eq.~(\ref{Cheby_ineq}) holds for the process $\left\{ X_i \right\}$. As the final preparation for bounds of cross-validation, let's first define the one-round Rademacher complexity for the cross validation.
%
%
\begin{definition}[One-round Rademacher complexity for cross-validation]
  Given the sample size $n$ and the number of folds $K$, the test size $n_s = n/K$ and the training size $n_t = 1 - n_s$. Let's denote $\left( Y_t^q, X_t^q \right)$ as the training set in $q$\textsuperscript{th} round and $\left( Y_s^q, X_s^q \right)$ as the test set in $q$\textsuperscript{th} round. We denote the \emph{one-round Rademacher complexity} of the model class $\Lambda$, $\mathrm{RC} \left( \Lambda, n, K \right)$ as
  \[
    \mathrm{RC} \left( \Lambda, n, K \right) = 
      \frac{1}{2} \mathrm{RC}_{n_s} \left( \Lambda \right)
      + \frac{1}{2} \mathrm{RC}_{n_t} \left( \Lambda \right)
  \]
\label{defn:one-round_RC}
\end{definition}
\noindent
A short comment need to address for definition~\ref{defn:one-round_RC}. As shown in the researches of learning theory, the Rademacher complexity is expected to bound the maximal deviation between the sample risk and population risk. Likwise, on average the one-round Rademacher complexity bounds the maximal difference between two sample risks, which is quite handy in the proof of the convoluted bounds for cross-validation in both i.i.d\ and non-i.i.d\ cases. Definition~\ref{defn:one-round_RC} completes the underpinnings to the proof of the convoluted PAC-bounds. Based on definition~\ref{defn:one-round_RC}, lemma~\ref{lem:Cheby_ineq} and theorem~\ref{thm:convoluted_VC_bound}, a convoluted Rademacher-bound may be established for the average test error in $K$-fold cross-validation.\footnote{Because it is simply a repetition of ordinary $K$-fold cross-validation, Theorem~\ref{thm:convoluted_VC_bound} may also be used to measure the performance of repeated $K$-fold cross-validation.}
%
%
\begin{theorem}[Convoluted Rademacher-bounds for the average test error in cross-validation]
Assume \textbf{A1} - \textbf{A5} are satisfied. Furthermore, assume Lemma~\ref{lem:Cheby_ineq} and Theorem~\ref{thm:one_round_VC_bound} hold.

\smallskip
  1. \textit{The $ \left\Vert \rho_j \right\Vert_{\Psi_2} \leqslant \infty $ case.}
  If \;$\left\Vert \rho_j \right\Vert_{\Psi_2} \leqslant \infty$, $\forall b \in \Lambda$, the following upper bound holds with probability at least $ \left( 1 - \kappa \right)^+$,

    \begin{align}
      \frac{ 1 }{ K } \sum^{ K }_{ q = 1 } \mathcal{R}_{n_s} \left( b, Y_s^q, X_s^q \right)
        \leqslant \frac{ 1 } { K } \sum^{ K }_{ q = 1 }
        \mathcal{R}_{n_t} \left( b, Y_t^q, X_t^q \right) 
          + 2 \cdot \mathrm{RC} \left( \Lambda, n, K \right)
          + \varsigma
      \label{eq1:convoluted_VC_bound}
    \end{align}

    \noindent
    where
    \[
      \kappa = \frac{ 1 + 2 V_K \left[ T_q \right] } { 2 \cdot \xi^2 \cdot K \cdot \log \left( 1 / \varpi \right) },
    \]
    $\xi$ and $\varsigma$ are parameters inherited from Theorem~\ref{thm:one_round_VC_bound}

\noindent
  2. \textit{The $ \left\Vert \rho_j \right\Vert_{\Psi_1} \leqslant \infty $ case}
  If \;$\left\Vert \rho_j \right\Vert_{\Psi_1} \leqslant \infty$ and $T_q$ still has a finite variance, the following upper bound holds with probability at least $ \left( 1 - \kappa \right)^+ $,

    \begin{align}
      \frac{ 1 }{ K } \sum^{ K }_{ q = 1 } \mathcal{R}_{n_s} \left( b, Y_s^q, X_s^q \right)
        \leqslant \frac{ 1 } { K } \sum^{ K }_{ q = 1 }
        \mathcal{R}_{n_t} \left( b, Y_t^q, X_t^q \right) 
          + 2 \cdot \mathrm{RC} \left( \Lambda, n, K \right)
          + \varsigma
      \label{eq2:convoluted_VC_bound}
    \end{align}
    where

    \[
      \kappa =
      \left\{
        \begin{array}{ll}
          \frac{ 6 \cdot \left( 1 + 2 V_K \left[ T_q \right] \right) } { K \cdot \log \sqrt[ c ] { 2 / \varpi } } ,
          & \mbox{ if } \varpi \in \left[ 2 \exp \left\{ -2c \right\}, 1 \right]. \\
          \frac{ 12 \cdot \left( 1 + 2 V_K \left[ T_q \right] \right) } { K \cdot \left( \log \sqrt[ c ] { 2 / \varpi } \right)^2 } ,
          & \mbox{ if } \varpi \in \left(0 , 2 \exp \left\{ -2c \right\} \right).
        \end{array}\right.
    \]
    $\varsigma$ is defined in the previous equation.
\label{thm:convoluted_VC_bound}
\end{theorem}

\begin{proof} 

\noindent
1. \textit{The $ \left\Vert \rho_j \right\Vert_{\Psi_2} \leqslant \infty $ case.}
\noindent
Since we have already defined
\begin{align}
  T_q  & = U_q - \mathbb{E} \left[ U_q \right] \notag \\
  & = \sup_{b \in \Lambda} \; \left\vert \mathcal{R}_{n_s} \left( b, Y_{s}^q, X_{s}^q \right) - \mathcal{R}_{n_t} \left( b, Y_{t}^q, X_{t}^q \right) \right\vert \\
  & \phantom{=} - \mathbb{E} \left[ \sup_{b \in \Lambda} \; \left\vert \mathcal{R}_{n_s} \left( b, Y_{s}^q, X_{s}^q \right) - \mathcal{R}_{n_t} \left( b, Y_{t}^q, X_{t}^q \right) \right\vert \right],
\end{align}

\noindent
Lemma~\ref{lem:Cheby_ineq} implies that
\begin{align}
  \mathrm{Pr} \left\{ \frac{1}{K} \sum^{K}_{q=1} T_q \geqslant \varsigma \right\}
  \leqslant
  \frac{ \gamma_0 \left[ T_q \right] }{ \varsigma^2 K} \cdot \left( 1 + 2 V_K \left[ T_q \right] \right)
\end{align}

\noindent
Based on Theorem~\ref{thm:one_round_VC_bound}, we let $\varsigma = 2 \cdot \xi \cdot \tilde{\sigma} \cdot \sqrt{ \frac{ \log \left( 1/\varpi \right) } { n / K } }$ when $\left\Vert \rho_j \right\Vert_{\Psi_2}$ is well defined. Since we can interchange $\sup \left\{ \cdot \right\}$ and $\mathbb{E} \left[ \cdot \right]$, the definition of $\gamma_0$ implies that
\begin{align}
  \gamma_0 \left[ T_q \right] & = \mathbb{E} \left[ U_q - \mathbb{E} \left[ U_q \right] \right]^2 \notag \\
  & \leqslant \mathbb{E} \left[
    \sup_{ b \in \Lambda} \; \left\vert
    \mathcal{R}_{n_s} \left( b, Y_s^q, X_s^q \right) - \mathcal{R}_{n_t} \left( b, Y_t^q, X_t^q \right)
    ~ \right\vert  \right]^2 \\
  & = \mathbb{E} \left[
    \sup_{ b \in \Lambda} \;  \left\{ \left( ~
    \mathcal{R}_{n_s} \left( b, Y_s^q, X_s^q \right) - \mathcal{R}_{n_t} \left( b, Y_t^q, X_t^q \right)
    ~ \right)^2 \right\} \right] \\
  & = \sup_{ b \in \Lambda} \left\{
    \mathbb{E} \left[
    \mathcal{R}_{n_s} \left( b, Y_s^q, X_s^q \right) - \mathcal{R}_{n_t} \left( b, Y_t^q, X_t^q \right)
    ~ \right]^2  \right\} \\
  & = \sup_{ b \in \Lambda} \left\{ \mathrm{var} \left[ ~ \mathcal{R}_{n_s} \left( b, Y_s^q, X_s^q \right) - \mathcal{R}_{n_t} \left( b, Y_t^q, X_t^q \right) ~ \right] \right\}.
\end{align}
\noindent
Since

\begin{align}
  \phantom{=} & ~ \mathrm{var} \left[ ~ \mathcal{R}_{n_s} \left( b, Y_s^q, X_s^q \right) - \mathcal{R}_{n_t} \left( b, Y_t^q, X_t^q \right) ~ \right] \notag \\
  = & ~ \mathrm{var} \left[ ~ \mathcal{R}_{n_s} \left( b, Y_s^q, X_s^q \right) - \mathcal{R} \left( b, Y, X \right) + \mathcal{R} \left( b, Y, X \right) - \mathcal{R}_{n_t} \left( b, Y_t^q, X_t^q \right) ~ \right] \\
  = & ~ \mathrm{var} \left[ \mathcal{R}_{n_s} \left( b, Y_s^q, X_s^q \right) - \mathcal{R} \left( b, Y, X \right) \right] + \mathrm{var} \left[ \mathcal{R} \left( b, Y, X \right) - \mathcal{R}_{n_t} \left( b, Y_t^q, X_t^q \right) \right],
\end{align}

\noindent
we can show that
\begin{align}
  \gamma_0 \left[ T_q \right] & \leqslant \sup_{ b \in \Lambda} \left\{ 2 \cdot \mathrm{var} \left[ \mathcal{R}_{n_s} \left( b, Y_s^q, X_s^q \right) - \mathcal{R} \left( b, Y, X \right) \right] \right\} \\
  & = \frac{2}{ n/K } \cdot \sup_{ b \in \Lambda} \left\{ \mathrm{var} \left[ Q \left( b, y^{i,j}, \mathbf{x}^{i,j} \right) - \mathcal{R} \left( b, Y, X \right) \right] \right\}
\end{align}

\noindent
Based on the definition of $\varsigma$,
\begin{align}
  \phantom{=} & ~
    \frac{ \gamma_0 \left[ T_q \right] }{ \varsigma^2 K} \cdot \left( 1 + 2 V_K \left[ T_q \right] \right)
  \notag \\
  \leqslant & ~
  \frac{ 2 / \left( n/K \right) \cdot \sup_{ b \in \Lambda} \left\{ \mathrm{var} \left[ Q \left( b, y^{i,j}, \mathbf{x}^{i,j} \right) - \mathcal{R} \left( b, Y, X \right) \right] \right\} \cdot \left( 1 + 2 V_K \left[ T_q \right] \right) }
  { \left( 4 \cdot \xi^2 \cdot K \right) / \left( n/K \right) \cdot \sup_{ b \in \Lambda} \left\{ \mathrm{var} \left[ Q \left( b, y^{i,j}, \mathbf{x}^{i,j} \right) - \mathcal{R} \left( b, Y, X \right) \right] \right\} \cdot \log \left( 1 / \varpi \right)} \\
  = & ~
  \frac{ 1 + 2 V_K \left[ T_q \right] } { 2 \cdot \xi^2 \cdot K \cdot \log \left( 1 / \varpi \right)}.
\end{align}

\noindent
Since $\mathbb{E} \left[ U_q \right] = \mathbb{E} \left[ \sup_{b \in \Lambda} \left\vert \mathcal{R}_{n_s} \left( b, Y_s^q, X_s^q \right) - \mathcal{R}_{n_t} \left( b, Y_t^q, X_t^q \right)\right\vert \right] \leqslant 2 \cdot \mathrm{RC} \left( \Lambda, n, K \right)$, $\forall b \in \Lambda$ the following equations hold,
\begin{align}
  \phantom{=} & \mathrm{Pr} \left\{ \frac{1}{K} \sum^{K}_{q=1} T_q \leqslant \varsigma \right\} \notag \\
  = & \mathrm{Pr} \left\{ \frac{ 1 }{ K } \sum^{ K }_{ q = 1 } \mathcal{R}_{n_s} \left( b, Y_s^q, X_s^q \right)
    \leqslant \frac{ 1 } { K } \sum^{ K }_{ q = 1 }
      \mathcal{R}_{n_t} \left( b, Y_t^q, X_t^q \right) 
      + \mathbb{E} \left[ U_q \right]
      + \varsigma \right\} \\
  = & \mathrm{Pr} \left\{ \frac{ 1 }{ K } \sum^{ K }_{ q = 1 } \mathcal{R}_{n_s} \left( b, Y_s^q, X_s^q \right)
    \leqslant \frac{ 1 } { K } \sum^{ K }_{ q = 1 }
      \mathcal{R}_{n_t} \left( b, Y_t^q, X_t^q \right) 
      + 2 \cdot \mathrm{RC} \left( \Lambda, n, K \right)
      + \varsigma \right\} \\
  \geqslant & \left( 1 - \frac{ 1 + 2 V_K \left[ T_q \right] } { 2 \cdot \xi^2 \cdot K \cdot \log \left( 1 / \varpi \right) } \right)^+.
\end{align}

\noindent
This completes the proof.\footnote{ Specifically, to make sure $1 - \left( 1 + 2 V_K \left[ T_q \right] \right) / \left( 2 \cdot \xi^2 \cdot K \cdot \log \left( 1 / \varpi \right) \right)$ is between $0$ and $1$, we need $ \left( 1 + 2 V_K \left[ T_q \right] \right) / \left( 2 \cdot \xi^2 \cdot K \cdot \log \left( 1 / \varpi \right) \right) \leqslant 1 $, which implies that $ 0 \leqslant \varpi \leqslant \exp \left\{ - \left( 1 + 2 V_K \left[ T_q \right] \right) / \left( 2 \cdot \xi^2 \cdot K \right) \right\}$.}
\bigskip


\noindent
2. \textit{The $ \left\Vert \rho_j \right\Vert_{\Psi_1} \leqslant \infty $ case}

We assume $ \left\Vert \rho_j \right\Vert_{\Psi_1}$ and $\mathrm{var} \left[ T_q \right]$ are well defined. If $\varpi \in \left[ 2 \exp \left\{ -2c \right\}, 1 \right]$, based on the definition of $\varsigma$ and the fact that 
\begin{align}
  \frac{ \gamma_0 \left[ T_q \right] }{ \varsigma^2 K} \cdot \left( 1 + 2 V_K \left[ T_q \right] \right)
  = & ~
  \frac{ 1 + 2 V_K \left[ T_q \right] } { K \cdot 2 \cdot \log \sqrt[ c ] { 2 / \varpi } } \cdot
    \frac{\gamma_0 \left[ T_q \right]}{\left\Vert \rho_j \right\Vert_{\Psi_1}^2}
\label{eq:p72}
\end{align}
\noindent
Since $\gamma_{0} \left[ T_q \right] / \left\Vert \rho_j \right\Vert_{\Psi_1}^2 \leqslant 8 + 4/(K-1)$, \footnote{ See lemma~\ref{lemma:a1} for detail}
\begin{align}
  \frac{ \gamma_0 \left[ T_q \right] }{ \varsigma^2 K} \cdot \left( 1 + 2 V_K \left[ T_q \right] \right)
  \leqslant & ~
  \frac{ 1 + 2 V_K \left[ T_q \right] } { K \cdot 2 \cdot \log \sqrt[ c ] { 2 / \varpi } } \cdot
    \left( 8 + \frac{4}{K-1} \right)  \\
  \leqslant & ~
  \frac{ 6 \cdot \left( 1 + 2 V_K \left[ T_q \right] \right) } { K \cdot \log \sqrt[ c ] { 2 / \varpi } }.
\end{align}
Likewise, if $\varpi \in \left(0 , 2 \exp \left\{ -2c \right\} \right)$,
\begin{align}
  \frac{ \gamma_0 \left[ T_q \right] }{ \varsigma^2 K} \cdot \left( 1 + 2 V_K \left[ T_q \right] \right)
  \leqslant & ~
  \frac{ 12 \cdot \left( 1 + 2 V_K \left[ T_q \right] \right) } { K \cdot \left( \log \sqrt[ c ] { 2 / \varpi } \right)^2 }.
\end{align}
\noindent
As a result, if we set
\begin{equation}
    \kappa =
      \left\{
        \begin{array}{ll}
          \frac{ 6 \cdot \left( 1 + 2 V_K \left[ T_q \right] \right) } { K \cdot \log \sqrt[ c ] { 2 / \varpi } } ,
          & \mbox{ if } \varpi \in \left[ 2 \exp \left\{ -2c \right\}, 1 \right]. \\
          \frac{ 12 \cdot \left( 1 + 2 V_K \left[ T_q \right] \right) } { K \cdot \left( \log \sqrt[ c ] { 2 / \varpi } \right)^2 } ,
          & \mbox{ if } \varpi \in \left(0 , 2 \exp \left\{ -2c \right\} \right).
        \end{array}\right.
\end{equation}
\begin{align}
  \mathrm{Pr} \left\{ \frac{ 1 }{ K } \sum^{ K }_{ q = 1 } \mathcal{R}_{n_s} \left( b, Y_s^q, X_s^q \right)
     \leqslant \frac{ 1 } { K } \sum^{ K }_{ q = 1 }
    \mathcal{R}_{n_t} \left( b, Y_t^q, X_t^q \right) + 2 \cdot \mathrm{RC} \left( \Lambda \right)
    + \varsigma \right\} \notag \\
    \geqslant \left( 1 - \kappa \right)^+,
\end{align}
\end{proof}

Similar to Theorem~\ref{thm:one_round_VC_bound}, the convoluted Rademacher-bounds in Theorem~\ref{thm:convoluted_VC_bound} measure the stability for the average test error in $K$ rounds of cross-validation. In a similar vein to Theorem~\ref{thm:one_round_VC_bound}, Theorem~\ref{thm:convoluted_VC_bound} reveals the encoded influence of $n$, $K$ and the tail heaviness of the loss distribution on the stability of the average test error. The new problem cross-validation brings, which does not occur with one-round validation, is that the training errors in each round are correlated at a constant level; moreover, the correlation gets larger as $K$ gets larger. This symptom occurs due to the overlapping of training sets for different rounds, which ultimately stems from the splitting scheme employed on the data. The autocorrelation among training errors from different rounds further weakens the stability of the evaluation result for cross-validation: the optimal model that cross-validation selects, which is ultimately trained with $n_t$ points, might behave very different in a new-coming sample of size $n_s$, if $K$ is unnecessarily large. The large magnitude of autocorrelation, brought up by a large $K$, will also make $1 - \kappa$, the probability approximation for eq.~(\ref{eq1:convoluted_VC_bound}) and~(\ref{eq2:convoluted_VC_bound}), less than or equal to zero, which makes the probabilistic approximation in Theorem~\ref{thm:convoluted_VC_bound} more difficult in finite samples. As a result, cross-validation is more complicated and $1 - \kappa$ depends on the correlation as well as on $K$. Hence, for given values of $n$ and the Orlicz-$\Psi_\nu$ norm of the loss, the sampling variation and correlation among training errors from different rounds makes the average test error very volatile. Thus, minimizing the average test error (the cross-sample mean) via tuning $K$ is difficult and potentially unfruitful.

Alternatively, from the perspective of minimax, the performance of cross-validation may be tuned by finding the minimum point on the convoluted Rademacher bounds parametrized by $K$ and $b$. Since the convoluted Rademacher bounds stand for the stability of the cross-validation under different values of $K$ and different $b \in \Lambda$, the best stability is acquired at the minimax-optimal pair of $ \left( K, b \right)$, denoted as $ \left( K^*, b^* \right)$. Thus, minimizing the convoluted Rademacher bounds guarantees the model selected by cross-validation at $K = K^*$ is the most stable across different samples sized $n_s$.

\section{Convoluted Rademacher-bounds for cross-validated test errors in the non-i.i.d. case}

In this section, we show that the results on constructing the i.i.d.\ convoluted Rademacher-bounds may be generalized to the non i.i.d.\ case. As an illustration, we consider $\beta$-mixing stationery processes. To ensure that the stability results for cross-validation are generalizable to other samples then, roughly speaking, the test and training sets must be independent and drawn from the sample population. However, if the data-generating process is dependent across time or space and the sample is a time series or spatially correlated, independence between the test set and training set will fail, potentially compromising the stability of the model evaluation results in cross-validation. As a result, the techniques in section~2 do not directly generalize to non-i.i.d.\ data. To solve the problem of dependence between training and test sets and construct convoluted Rademacher-bounds for non-i.i.d.\ data, we introduce the technique of `independent blocks' \citep{bernstein1927extension, yu1994rates}.
\noindent
\begin{figure}
  \centering
  \subfloat[\label{fig:IB1} original data]
  {\includegraphics[width=0.42\paperwidth]{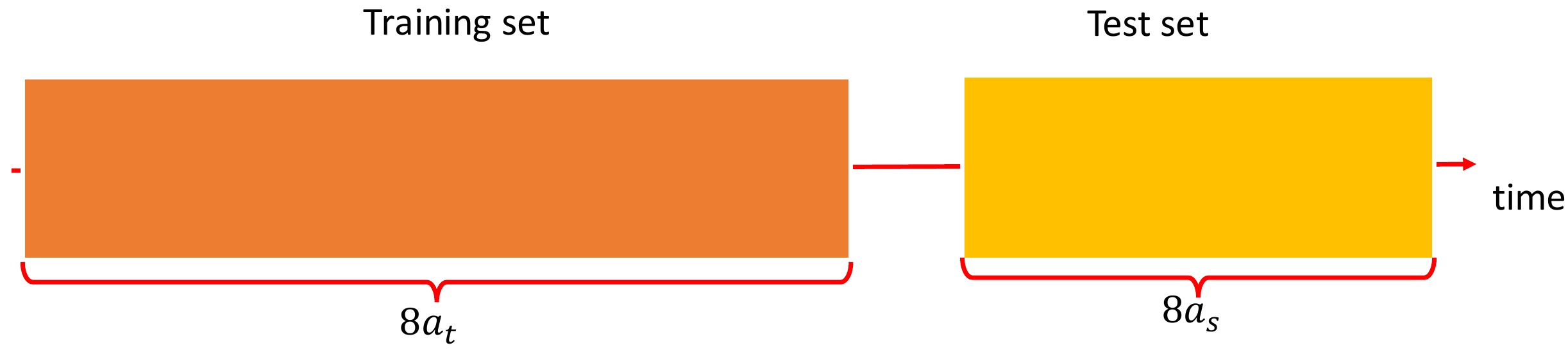}}

  \subfloat[\label{fig:IB2} blocking on original data]
  {\includegraphics[width=0.42\paperwidth]{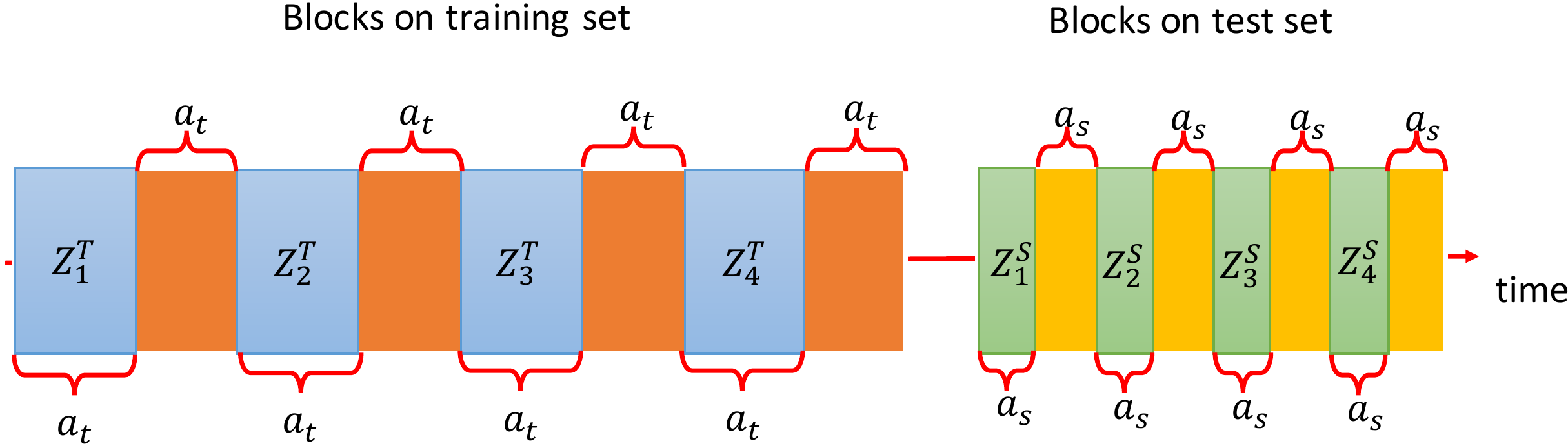}}

  \subfloat[\label{fig:IB3} independent blocks]
  {\includegraphics[width=0.42\paperwidth]{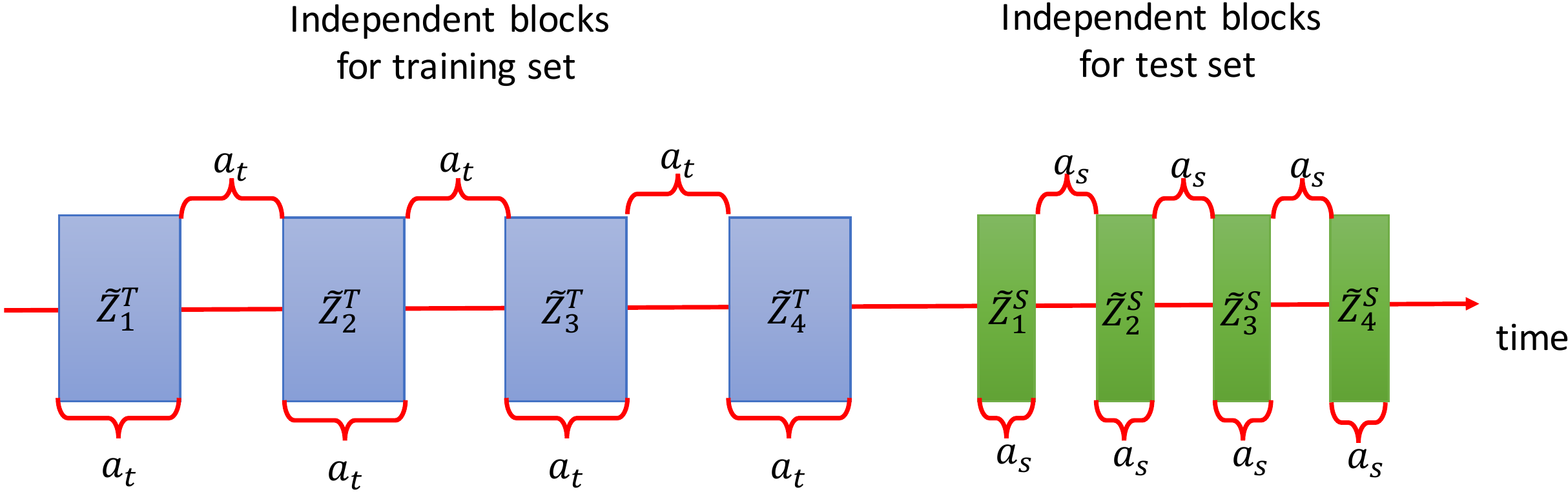}}

  \caption{illustration of independent blocks techinique where $\mu = 4$}
  \label{fig:IB}
\end{figure}

As shown in Figure~\ref{fig:IB}, the essence of the independent blocks technique is approximation. Precisely speaking, the technique consists of first splitting a sequence of training data, for example $S^T$, into two subsequences $S^T_0$ and $S^T_1$, each made of $\mu$ blocks of $a_t$ consecutive points. Likewise, the test data $S^S$ is also split into $S^S_0$ and $S^S_1$, each made of $\mu$ blocks of $a_s$ consecutive points. Given a sequence $S^T = \left(z^T_1, z^T_2, \ldots, z_{n_t}\right)$ with $n_t = 2 a_t \mu$ and $S^S = \left(z^S_1, z^S_2, \ldots, z_{n_s}\right)$ with $n_s = 2 a_s \mu$. $S^T_0$ and $S^S_0$ are defined as follows
\noindent
\begin{eqnarray}
  S^T_0 = & \left( Z^T_1, Z^T_2, \ldots, Z^T_{\mu}\right), & \mbox{ where } Z^T_i = \left( z^T_{(2i-1)+1}, \ldots , z^T_{(2i-1)+a}\right) \\
  S^S_0 = & \left( Z^S_1, Z^S_2, \ldots, Z^S_{\mu}\right), & \mbox{ where } Z^S_i = \left( z^S_{(2i-1)+1}, \ldots , z^S_{(2i-1)+a}\right)
\end{eqnarray}
\noindent
If $\left\{z_i\right\}$ is $\beta$-mixing and the mixing coefficient decays rapidly (exponentially or algebraically), large $a_t$ and $a_s$ make each block in $S^T_0$ or $S^S_0$ `almost' independent from each other while each block is still drawn from the same distribution as $\left\{z_i\right\}$ in the population. Next, we create two new sequences of blocks: $\widetilde{S}^T = \left( \widetilde{Z}^T_1, \ldots, \widetilde{Z}^T_\mu \right)$ and $\widetilde{S}^S = \left( \widetilde{Z}^S_1, \ldots, \widetilde{Z}^S_\mu \right)$. The key feature of $\widetilde{S}^T$ is that, unlike in $S^T_0$,  $\widetilde{Z}^T_i$ and $\widetilde{Z}^T_j$ are, for any $i,j$, strictly independent. Also, $\widetilde{Z}^T_i$ is still identically distributed as $Z^T_i$ for any $i$. Hence, the technique creates a sequence of independent, equally-sized blocks, to which standard i.i.d.\ techniques may be applied. Likewise, all blocks in $\widetilde{S}^S$ are independent and $\widetilde{Z}^S_i$ is identically distributed as $Z^S_i$ for any $i$. Also, $\widetilde{Z}^T_i$ is independent with $\widetilde{Z}^S_j$, for any $i$ and $j$. This ensures the independent blocks for the training data and test data are mutually independent. Since the original blocks, $S^T_0$ and $S^S_0$, are `very similar' to their independent blocks, $\widetilde{S}^T_0$ and $\widetilde{S}^S_0$, we can approximate the performance of cross-validation on $\left(S^T_0\;,\;S^S_0\right)$ by its performance on $\left(\widetilde{S}^T_0\;,\;\widetilde{S}^S_0\right)$. The following theorem by \citet{yu1994rates} illustrates that the independent blocks and original blocks are similar.


\begin{theorem}[\citet{yu1994rates}]
Let $\mu > 1$ and suppose that $\Lambda$ is a model class bounded by $M \geqslant 0$, defined on the blocks $\left( S^T_0\;,\;S^S_0 \right)$. Assume that $h$ is real-valued and measurable in the product probability space $\left(\Pi_{i=1}^{\mu} \Omega_i, \Pi_{i=1}^{\mu} \mathcal{F}_i \right)$. Then, for any  $S^T_0$ and $S^S_0$ drawn from a stationery $\beta$-mixing process,
  \begin{eqnarray}
    \left\vert \mathbb{E}_{S^T_0} \left[ h \right] - \mathbb{E}_{\widetilde{S}^T} \left[ h \right]\right\vert
    & \leqslant & \left( \mu - 1 \right) M \beta_{a_t} \\
    \left\vert \mathbb{E}_{S^S_0} \left[ h \right] - \mathbb{E}_{\widetilde{S}^S} \left[ h \right]\right\vert
    & \leqslant & \left( \mu - 1 \right) M \beta_{a_s}
  \end{eqnarray}
where $\mathbb{E}_{S^T_0}$ and $\mathbb{E}_{S^S_0}$ is the expectation w.r.t.\ $S^T_0$ and $S^S_0$; $\mathbb{E}_{\widetilde{S}^T}$ and $\mathbb{E}_{\widetilde{S}^S}$ is the expectation w.r.t.\ $\widetilde{S}^T$ and $\widetilde{S}^S$.
\label{thm:Yu94}
\end{theorem}

\noindent
Theorem~\ref{thm:Yu94} is the key result that allows us to approximate the performance of cross-validation with independent blocks. Based on Theorem~\ref{thm:Yu94} and the McDiarmid inequality, we may derive the Rademacher bound for the one-round test error in cross-validation as follows.


\begin{theorem}
Assume that the test and training data are both sampled from a stationery $\beta$-mixing process. Assume also that the loss function $Q$ is measurable in space $\left(\Omega, \mathcal{F}\right)$, space $\left(\Omega_i, \mathcal{F}_i\right),\;\forall i$, and product space $\left( \Pi_{i=1}^{\mu} \Omega_i, \Pi_{i=1}^{\mu} \mathcal{F}_i \right)$. Denote $\varpi' = \varpi - \left( \mu - 1 \right) \left[ \beta_{a_t} + \beta_{a_s} \right]$. If $\left( \mu - 1 \right) \left[ \beta_{a_t} + \beta_{a_s} \right] < 1$, the following Rademacher bound holds with probability at least $ 1 - \varpi $, $\forall \varpi \in \left( \; \left( \mu - 1 \right) \left[ \beta_{a_t} + \beta_{a_s} \right] \; , \; 1 \; \right]$,
  \begin{align}
    \mathcal{R}_{n_s} \left( b, X_s, Y_s \right) 
    \leqslant 
    \mathcal{R}_{n_t} \left( b, X_t, Y_t \right) +  M \cdot \sqrt{ \frac{ \log \left( 4 / \varpi' \right) } { 2 \mu } }  + 2 \cdot \mathrm{RC}_{\widetilde{S}^S_0} \left( \Lambda \right)
  \label{eqn:one-round_Rademacher_CV}
  \end{align}
where $\mathrm{RC}_{\widetilde{S}^S_0} \left( \Lambda \right)$ is the Rademacher bound of the model class $\Lambda$ on the block $\widetilde{S}^S_0$.
\label{thm:one_round_RC_bound}
\end{theorem}

%
\begin{proof}
In this theorem we quantify the probability of
\[
  \mathrm{Pr} \left( \sup_{ b \in \Lambda } \left\vert \mathcal{R}_{n_t} \left( b, X_t, Y_t \right) - \mathcal{R}_{n_s} \left( b, X_s, Y_s \right) \right\vert \geqslant \epsilon \right).
\]
Based on the independent blocks,
\begin{align}
  \phantom{=} & \mathrm{Pr} \left( \sup_{ b \in \Lambda } \left\vert \mathcal{R}_{n_t} \left( b, X_t, Y_t \right) - \mathcal{R}_{n_s} \left( b, X_s, Y_s \right) \right\vert \geqslant \epsilon \right) \notag \\
  = & ~ \mathrm{Pr} \left( \sup_{ b \in \Lambda } \left\vert \frac{1}{2} \left[ \mathcal{R} \left( b, S^T_0 \right) + \mathcal{R} \left( b, S^T_1 \right) \right] - \frac{1}{2} \left[ \mathcal{R} \left( b, S^S_0 \right) + \mathcal{R} \left( b, S^S_1 \right) \right] \right\vert \geqslant \epsilon \right) \\
  = & ~ \mathrm{Pr} \left( \sup_{ b \in \Lambda } \left\vert \left[ \mathcal{R} \left( b, S^T_0 \right) - \mathcal{R} \left( b, S^S_0 \right) \right] + \left[ \mathcal{R} \left( b, S^T_1 \right) - \mathcal{R} \left( b, S^S_1 \right) \right] \right\vert \geqslant 2\epsilon \right) \\
  \leqslant & ~ \mathrm{Pr} \left( \sup_{ b \in \Lambda } \left\vert \mathcal{R} \left( b, S^T_0 \right) - \mathcal{R} \left( b, S^S_0 \right) \right\vert + \sup_{ b \in \Lambda } \left\vert \mathcal{R} \left( b, S^T_1 \right) - \mathcal{R} \left( b, S^S_1 \right) \right\vert \geqslant 2\epsilon \right) \\
  \leqslant & ~ \mathrm{Pr} \left( \sup_{ b \in \Lambda } \left\vert \mathcal{R} \left( b, S^T_0 \right) - \mathcal{R} \left( b, S^S_0 \right) \right\vert \geqslant \epsilon \right) + \mathrm{Pr} \left( \sup_{ b \in \Lambda } \left\vert \mathcal{R} \left( b, S^T_1 \right) - \mathcal{R} \left( b, S^S_1 \right) \right\vert \geqslant \epsilon \right). \label{eqn:iid}
\end{align}
\noindent
Since $\left( S^S_0, S^T_0 \right)$ and $\left( S^S_1, S^T_1 \right)$ are identically distributed due to stationarity,
\begin{align}
  \mathrm{Pr} \left( \sup_{ b \in \Lambda } \left\vert \mathcal{R} \left( b, S^T_0 \right) - \mathcal{R} \left( b, S^S_0 \right) \right\vert \geqslant \epsilon \right)
  =
  \mathrm{Pr} \left( \sup_{ b \in \Lambda } \left\vert \mathcal{R} \left( b, S^T_1 \right) - \mathcal{R} \left( b, S^S_1 \right) \right\vert \geqslant \epsilon \right).
\end{align}
\noindent
This implies, alongside eq.~(\ref{eqn:iid}), that
\begin{align}
  \mathrm{Pr} \left( \sup_{ b \in \Lambda } \left\vert \mathcal{R}_{n_t} \left( b, X_t, Y_t \right) - \mathcal{R}_{n_s} \left( b, X_s, Y_s \right) \right\vert \geqslant \epsilon \right)
  \leqslant
  2 \cdot \mathrm{Pr} \left( \sup_{ b \in \Lambda } \left\vert \mathcal{R} \left( b, S^T_0 \right) - \mathcal{R} \left( b, S^S_0 \right) \right\vert \geqslant \epsilon \right)
\end{align}
Denote
\begin{align}
  \Phi \left( S^S_0 \right) & = \sup_{ b \in \Lambda} \left\vert \mathcal{R} \left(b\; X, \; Y \right) - \mathcal{R}_{n_s} \left(b,\;S^S_0 \right) \right\vert, \notag \\
  \Phi \left( S^T_0 \right) & = \sup_{ b \in \Lambda} \left\vert \mathcal{R} \left(b\; X, \; Y \right) - \mathcal{R}_{n_t} \left( b,\; S^T_0 \right)\right\vert. \notag
\end{align}
As a result, $\forall \epsilon > 0$,
\begin{align}
  \mathrm{Pr} \left( \sup_{ b \in \Lambda } \left\vert \mathcal{R}_{n_t} \left( b, \; S^T_0 \right) - \mathcal{R}_{n_s} \left( b, \; S^S_0 \right) \right\vert \geqslant \epsilon \right)
  \leqslant
  \mathrm{Pr} \left( \Phi \left( S^T_0 \right) + \Phi \left( S^S_0 \right) \geqslant \epsilon \right)
\end{align}

To approximate the probability of $\mathrm{Pr} \left( \; \Phi \left( S^T_0 \right) + \Phi \left( S^S_0 \right) \geqslant \epsilon \right) $, we define $\epsilon_1 = \epsilon / 2 - \mathbb{E} \left[ \Phi \left( \widetilde{S}^T_0 \right) \right]$ and $\epsilon_2 = \epsilon / 2 - \mathbb{E} \left[ \Phi \left( \widetilde{S}^S_0 \right) \right]$. Hence, $\forall \epsilon / 2 > \max \left( \mathbb{E} \left[ \Phi \left( \widetilde{S}^T_0 \right) \right], \mathbb{E} \left[ \Phi \left( \widetilde{S}^S_0 \right) \right] \right)$,
\begin{align}
  \mathrm{Pr} \left( \Phi \left( S^T_0 \right) + \Phi \left( S^S_0 \right) \geqslant \epsilon \right)
  & \leqslant \mathrm{Pr} \left( \Phi \left( S^T_0 \right) \geqslant \epsilon / 2 \right) + \mathrm{Pr} \left( \Phi \left( S^S_0 \right) \geqslant \epsilon / 2 \right) \notag \\
  & = \mathrm{Pr} \left( \Phi \left( S^T_0 \right) - \mathbb{E} \left[ \Phi \left( \widetilde{S}^T_0 \right) \right] \geqslant \epsilon_1 \right) \notag \\
  & \phantom{=} \, + \mathrm{Pr} \left( \Phi \left( S^S_0 \right) - \mathbb{E} \left[ \Phi \left( \widetilde{S}^S_0 \right) \right] \geqslant \epsilon_2 \right).
\end{align}
\noindent
Since the probability may be considered to be the expectation of some indicator function, we can apply Theorem~\ref{thm:Yu94}. Thus,
\begin{align}
  \mathrm{Pr} \left( \Phi \left( S^T_0 \right) - \mathbb{E} \left[ \Phi \left( \widetilde{S}^T_0 \right) \right] \geqslant \epsilon_1 \right) & \leqslant \mathrm{Pr} \left( \Phi \left( \widetilde{S}^T_0 \right) - \mathbb{E} \left[ \Phi \left( \widetilde{S}^T_0 \right) \right] \geqslant \epsilon_1 \right) \notag \\
  & \phantom{=} + \left( \mu - 1 \right) \beta_{a_t},
  \label{eqn:IB_app_t} \\
  \smallskip
  \mathrm{Pr} \left( \Phi \left( S^S_0 \right) - \mathbb{E} \left[ \Phi \left( \widetilde{S}^S_0 \right) \right] \geqslant \epsilon_2 \right) & \leqslant \mathrm{Pr} \left( \Phi \left( \widetilde{S}^S_0 \right) - \mathbb{E} \left[ \Phi \left( \widetilde{S}^S_0 \right) \right] \geqslant \epsilon_2 \right) \notag \\
  & \phantom{=} + \left( \mu - 1 \right) \beta_{a_s}.
  \label{eqn:IB_app_s}
\end{align}

By applying the McDiarmid inequality to the RHS of eqs.~(\ref{eqn:IB_app_t}) and~(\ref{eqn:IB_app_s}), we get the exponential inequalities for the LHS of eqs.~(\ref{eqn:IB_app_t}) and~(\ref{eqn:IB_app_s}), which yields
\begin{align}
  \mathrm{Pr} \left( \Phi \left( S^T_0 \right) - \mathbb{E} \left[ \Phi \left( \widetilde{S}^T_0 \right) \right] \geqslant \epsilon_1 \right) & \leqslant \exp \left( - \frac{ 2 \mu \left( \epsilon_1 \right)^2 }{ M^2 } \right) + \left( \mu - 1 \right) \beta_{a_t}, \\
  \mathrm{Pr} \left( \Phi \left( S^S_0 \right) - \mathbb{E} \left[ \Phi \left( \widetilde{S}^S_0 \right) \right] \geqslant \epsilon_2 \right) & \leqslant \exp \left( - \frac{ 2 \mu \left( \epsilon_2 \right)^2 }{ M^2 } \right) + \left( \mu - 1 \right) \beta_{a_s}.
\end{align}

\noindent
By denoting $\widetilde{\epsilon} = \min \left( \epsilon_1, \epsilon_2 \right) =  \epsilon / 2 - \max \left( \mathbb{E} \left[ \Phi \left( \widetilde{S}^S_0 \right) \right], \mathbb{E} \left[ \Phi \left( \widetilde{S}^T_0 \right) \right] \right)$,
\begin{align}
  \mathrm{Pr} \left( \sup_{ b \in \Lambda } \left\vert \mathcal{R}_{n_t} \left( b, X_t, Y_t \right) - \mathcal{R}_{n_s} \left( b, X_s, Y_s \right) \right\vert \geqslant \epsilon \right)
  \leqslant 4 \exp \left( - \frac{ 2 \mu \left(\widetilde{\epsilon} \right)^2 }{ M^2 } \right) + 2 \left( \mu - 1 \right) \left[ \beta_{a_t} + \beta_{a_s} \right],
\end{align}
\noindent
which yields
\begin{align}
  \mathrm{Pr} \left( \sup_{ b \in \Lambda } \left\vert \mathcal{R}_{n_t} \left( b, X_t, Y_t \right) - \mathcal{R}_{n_s} \left( b, X_s, Y_s \right) \right\vert \leqslant \epsilon \right)
  \geqslant 1 - \left[ 4 \exp \left( - \frac{ 2 \mu \left( \widetilde{\epsilon} \right)^2}{M^2} \right)
    + 2 \left( \mu - 1 \right) \left[ \beta_{a_t} + \beta_{a_s} \right] \right].
\label{eqn:quarter-result_one_round}
\end{align}

\noindent
If we set $ \varpi= 4 \exp\left( - \frac{ 2 \mu \left( \widetilde{\epsilon} \right)^2 }{ M^2 } \right) + 2 \left( \mu - 1 \right) \left[ \beta_{a_t} + \beta_{a_s} \right]$, \, $\varpi' = \varpi - \left( \mu - 1 \right) \left[ \beta_{a_t} + \beta_{a_s} \right]$ and assume $\varpi' > 0$,
\begin{align}
  \epsilon & = M \cdot \sqrt{ \frac{ \log \left( 4 / \varpi' \right) } { 2\mu } }  + 2 \cdot \max \left( \mathbb{E} \left[ \Phi \left( \widetilde{S}^T_0 \right) \right], \mathbb{E} \left[ \Phi \left( \widetilde{S}^S_0 \right) \right] \right) \\
  & \leqslant M \cdot \sqrt{ \frac{ \log \left( 4 / \varpi' \right) } { 2\mu } }  + 2 \cdot \max \left( \mathrm{RC}_{\widetilde{S}^T_0} \left( \Lambda\right) , \mathrm{RC}_{\widetilde{S}^S_0} \left( \Lambda\right) \right) \\
  & = M \cdot \sqrt{ \frac{ \log \left( 4 / \varpi' \right) } { 2\mu } }  + 2 \cdot \mathrm{RC}_{\widetilde{S}^S_0} \left( \Lambda \right).
\end{align}

\noindent
As a result, eq.~(\ref{eqn:quarter-result_one_round}) may be respecified as, $\forall b \in \Lambda$,
\begin{equation}
  \mathrm{Pr} \left( \mathcal{R}_{n_s} \left( b, X_s, Y_s \right) \leqslant \mathcal{R}_{n_t} \left( b, X_t, Y_t \right) +  M \cdot \sqrt{ \frac{ \log \left( 4 / \varpi' \right) } { 2\mu } }  + 2 \cdot \mathrm{RC}_{\widetilde{S}^S_0} \left( \Lambda \right) \right)
  \geqslant 1 - \varpi
\label{eqn:semi-result_one_round}
\end{equation}
\end{proof}

As shown in Theorem~\ref{thm:one_round_RC_bound}, a Rademacher bound for the one-round test error is constructed to quantify the variation of one-round test error, similar to Theorem~\ref{thm:one_round_VC_bound}. The larger the probability the i.i.d.\ bound for the one-round test error holds, the higher the bound. However, the probability the non-i.i.d.\ bound for one-round test error holds is constrained by the magnitude of the $\beta$-mixing coefficient. In a similar vein, based on Lemma~\ref{lem:Cheby_ineq}, we may construct the convoluted Rademacher-bound for cross-validation in the non-i.i.d.\ case.


\begin{theorem}
Assume Lemma~\ref{lem:Cheby_ineq} and Theorem~\ref{thm:one_round_RC_bound} hold. With probability at least $1 - \frac{ 2 \left( 1 + 2V_k \left[ T_q \right] \right) } { \log \left( 4/ \varpi' \right) K / \mu } $, the following bound holds
  \begin{align}
    \frac{1}{K} \sum_{q = 1}^{K} \mathcal{R}_{n_s} \left( b, X_s^q, Y_s^q \right)
    \leqslant
    \frac{1}{K} \sum_{q = 1}^{K} \mathcal{R}_{n_t} \left( b, X_t^q, Y_t^q \right)
    + 2 \cdot \mathrm{RC}_{ \widetilde{S}^S_0 } \left( \Lambda \right)
    + M \cdot \sqrt{ \frac{ \log \left( 4 / \varpi' \right) } { 2\mu } }
  \label{eqn:CV_RC_bound}
  \end{align}
\noindent
where
  \begin{align}
    \varpi' \in
     \left(0, 4 \exp \left\{ - \frac{ 2 \left( 1 + 2 V_K \left[ T_q \right] \right) } { K / \mu } \right\} \right]
  \end{align}
to make sure $1 - \frac{ 2 \left( 1 + 2V_k \left[ T_q \right] \right) } { \log \left( 4/ \varpi' \right) K / \mu } \in \left( 0 , 1 \right]$.
\label{thm:RC_bound_CV}
\end{theorem}

\begin{proof}
Since we have already defined
\begin{align}
  T_q  = &\sup_{b \in \Lambda} \; \left\vert \mathcal{R}_{n_s} \left( b, Y_{s}^q, X_{s}^q \right) - \mathcal{R}_{n_t} \left( b, Y_{t}^q, X_{t}^q \right) \right\vert \notag \\
  & - \mathbb{E} \left[ \sup_{b \in \Lambda} \; \left\vert \mathcal{R}_{n_s} \left( b, Y_{s}^q, X_{s}^q \right) - \mathcal{R}_{n_t} \left( b, Y_{t}^q, X_{t}^q \right) \right\vert \right],
\end{align}

\noindent
Lemma~\ref{lem:Cheby_ineq} implies that
\begin{align}
  \mathrm{Pr} \left\{ \frac{1}{K} \sum^{K}_{q=1} T_q \geqslant \varsigma \right\}
  \leqslant
  \frac{ \gamma_0 \left[ T_q \right] }{ \varsigma^2 K} \cdot \left( 1 + 2 V_K \left[ T_q \right] \right)
\end{align}

\noindent
Based on Theorem~\ref{thm:one_round_RC_bound}, we let $\varsigma = M \cdot \sqrt{ \frac{ \log \left( 4 / \varpi' \right) } { 2\mu } }$, which implies that
\begin{align}
  \phantom{=} & ~
    \frac{ \gamma_0 \left[ T_q \right] }{ \varsigma^2 K} \cdot \left( 1 + 2 V_K \left[ T_q \right] \right)
  \notag \\
  = & ~
  \frac{ \gamma_0 \left[ T_q \right] } { M^2 \cdot K \cdot \log \left( 4 / \varpi' \right) / \left( 2\mu \right) } \cdot \left( 1 + 2 V_K \left[ T_q \right] \right) \\
  \leqslant & ~
  \frac{ 2 \left( 1 + 2 V_K \left[ T_q \right] \right) } { \log \left( 4 / \varpi' \right) \cdot K / \mu }.
\end{align}

\noindent
As a result,
\begin{align}
  \mathrm{Pr} \left\{ \frac{1}{K} \sum_{q = 1}^{K} \mathcal{R}_{n_s} \left( b, X_s^q, Y_s^q \right)
  \leqslant
  \frac{1}{K} \sum_{q = 1}^{K} \mathcal{R}_{n_t} \left( b, X_t^q, Y_t^q \right)
  + 2 \cdot \mathrm{RC}_{ \widetilde{S}^S_0 } \left( \Lambda \right)
  + M \cdot \sqrt{ \frac{ \log \left( 4 / \varpi' \right) } { 2\mu } } \right\} \notag \\
  \geqslant 1 - \frac{ 2 \left( 1 + 2 V_K \left[ T_q \right] \right) } { \log \left( 4 / \varpi' \right) \cdot K / \mu }
\end{align}

\noindent
To make sure $1 - 2\left( 1 + 2 V_K \left[ T_q \right] \right) / \left( \log \left( 4 / \varpi' \right) \cdot K / \mu \right)$ is between $0$ and $1$, we need
\begin{align}
  \frac{ 2 \left( 1 + 2 V_K \left[ T_q \right] \right) }
    { \log \left( 4 / \varpi' \right) \cdot K / \mu }
  \leqslant 1,
\end{align}

\noindent
which implies that

\noindent
\begin{align}
  \varpi' \in
   \left( 0, 4 \exp \left\{ - \frac{ 2 \left( 1 + 2 V_K \left[ T_q \right] \right) } { K / \mu } \right\} \right].
\end{align}
\end{proof}

\section{Conclusion}

In this paper, we construct the general theory for cross-validation from the perspective of the stability of the model evaluation/selection result. We propose the first class of upper bounds, referred to as the convoluted Rademacher-bounds, to quantify the stability of cross-validation in both i.i.d.\ and non-i.i.d.\ cases. Given the fact that learning tasks involving subsampling or resampling (like cross-validation, bootstrap or jackknife) may result in substantial variation due to the sampling scheme, we consider the hyperparameter tuning scheme (optimizating the performance of hyperparameter) from the perspective of stability. This approach reveals insights around model selection from a stability perspective and enlightens the learning theory for regularizing complicated and hierarchical learning algorithms.

The practical difficulty in formally translating these idea is straightforward. Our goal is to establish (probablistic) uniform bounds for the learning algorithm in some given functional space. As a result, the techniques have to deal with maximal deviation, or $\sup_{b \in \Lambda} \left\vert \cdot \right\vert$, similar to the work of Vapnik. However, Rademacher learning theory carries an advantage in terms of conciseness, tightness and utilization, which can be exploited with data. These advantages distinguish Rademacher theory from VC theory, necessitating a different path to Vapnik's work. The key is utilizing Orlicz-$\Psi_\nu$ space, imported from functional analysis and harmonic analysis.

Due to the distinctions between Orlicz-Young theory and the classical integration theory of Lebesgue-Darboux-Jordan, a caveat of our approach is also worth mentioning. In one-round Rademacher bounds, the concentration inequality based on the Orlicz-$\Psi_\nu$ norms shows its power by quantifying the exponential concentration tendency without requiring the Lebesgue $p$-norm. However, due to the fact that the classical techniques of studying time series typically need a well-defined Lebesgue $2$-norm, we have to assume $T_q$ has a finite Lebesgue $2$-norm, which weakens the power of Orlicz-$\Psi$ space. Besides, the independent blocks require a well-defined envelope function for the random variable of interest, which makes the power independent blocks confined temporarily on the bounded variables

Our next task is to quantify, in the non-i.i.d.\ case, the exponential concentration tendency in Orlicz-$\Psi_\nu$ space, which does not require a finite Lebesgue $2$-norm for random variables. We are also working on the methodology of studying the LLN and CLT paradigm for time series without a finite Lebesgue $2$-norm. These two methods can ultimately free the power of the Orlicz-$\Psi_\nu$ space, and it will also generalize the properties of cross-validation studied in this paper to the flat/long/heavy-tailed distributions, in both i.i.d.\ and non-i.i.d.\ cases.

\acks{We would like to thank Prof. Pierre Del Moral, Prof. Peter Hall, Prof. Mike Bain and Yi-Lung Chen for valuable comments on earlier drafts. He also would like to acknowledge seminar participants at NICTA, UNSW and Uni Melbourne for useful questions and comments. Fisher would like to acknowledge the financial support of the Australian Research Council grant DP0663477.}

\appendix
\section{Supplementary materials for mathematical proof}
%
%
\begin{lemma}
  Let's denote $\gamma_0 \left[ \cdot \right]$ as the variance operator for some random variable, we also denote $\left\Vert \cdot \right\Vert_{\Psi_1}$ as the Orlicz-$\Psi_1$ norm for some random variable. The following statement holds.
  \begin{equation}
  \frac{ \gamma_0 \left[ T_q \right] } { \left\Vert \rho_j \right\Vert _{\Psi_1}^2 } 
    \leqslant 8 + \frac{ 4 }{ K - 1 }
  \end{equation}
\label{lemma:a1}
\end{lemma}
%
%
\begin{proof}
Based on the definition, 
  \begin{align}
    \frac{ \gamma_0 \left[ T_q \right] }
      { \left\Vert \rho_j \right\Vert _{\Psi_1}^2 } 
      & = \frac{ \mathrm{var} \left[ \sup_{ b \in \Lambda } \left\vert \mathcal{R}_{n_t} - \mathcal{R}_{n_s} \right\vert \right] }
        { \left\Vert \rho_j \right\Vert_{\Psi_1}^2 } \\
    & \leqslant \frac{ \mathbb{E} \left[ \sup_{ b \in \Lambda } \left\vert \mathcal{R}_{n_t} - \mathcal{R}_{n_s} \right\vert \right]^2} 
        { \left\Vert \rho_j \right\Vert_{\Psi_1}^2 }.
  \end{align}
\noindent
Since the norm is convex,
  \begin{align}
    \frac{ \mathbb{E} \left[ \sup_{ b \in \Lambda } \left\vert \mathcal{R}_{n_t} - \mathcal{R}_{n_s} \right\vert \right]^2} 
        { \left\Vert \rho_j \right\Vert_{\Psi_1}^2 }
    & \leqslant \frac{ \mathbb{E} \left[ \sup_{ b \in \Lambda } \left\vert \mathcal{R}_{n_t} - \mathcal{R} \right\vert 
        + \sup_{ b \in \Lambda } \left\vert \mathcal{R}_{n_s} - \mathcal{R} \right\vert \right]^2 } 
        {\left\Vert \rho_j \right\Vert_{\Psi_1}^2 }  \\
    & = \frac{ \mathbb{E} \left[ \sup_{ b \in \Lambda } \left( \mathcal{R}_{n_s} - \mathcal{R} \right)^2 \right] 
          + \mathbb{E} \left[ \sup_{ b \in \Lambda } \left( \mathcal{R}_{n_t} - \mathcal{R} \right)^2 \right] } { \left\Vert \rho_{j}\right\Vert _{\Psi_{1}}^{2}} \notag \\
    & \phantom{\leqslant} + \frac{ 2 \mathbb{E} \left[ \sup_{ b \in \Lambda } \left( \mathcal{R}_{n_s} - \mathcal{R} \right) \right] \cdot \mathbb{E} \left[ \sup_{ b \in \Lambda } \left( \mathcal{R}_{n_t} - \mathcal{R} \right) \right] } { \left\Vert \rho_j \right\Vert _{\Psi_1}^2 }.
    \label{eq:proof_a1_1}
  \end{align}
Since the $\sup \left[ \cdot \right]$ and $\mathbb{E} \left[ \cdot \right]$ are interchangable, eq.~(\ref{eq:proof_a1_1}) implies that
  \begin{align}
    \mbox{ eq.~(\ref{eq:proof_a1_1})} & = \frac{ \sup_{ b \in \Lambda } \left\{ \mathbb{E} \left[ \left( \mathcal{R}_{n_s} - \mathcal{R} \right)^2 \right] \right\} + \sup_{ b \in \Lambda } \left\{ \mathbb{E} \left[ \left( \mathcal{R}_{n_t} - \mathcal{R} \right)^2 \right] \right\} } { \left\Vert \rho_j \right\Vert _{\Psi_1}^2 } \\
    & \phantom{ \leqslant } + \frac{ 2 \sup_{ b \in \Lambda } \left\{ \mathbb{E} \left[ \left( \mathcal{R}_{n_s} - \mathcal{R} \right) \right] \right\} 
        \cdot \sup_{ b \in \Lambda } \left\{ \mathbb{E} \left[ \left( \mathcal{R}_{n_t} - \mathcal{R} \right) \right] \right\} }
        { \left\Vert \rho_j \right\Vert_{ \Psi_1}^2 } \\
    & \leqslant \frac{ 2 \sup_{ b \in \Lambda } \left\{ \mathrm{var} \left[ \mathcal{R}_{n_s} - \mathcal{R} \right] \right\} 
        + \sup_{b\in\Lambda} \left\{ \mathrm{var} \left[\mathcal{R}_{n_t} - \mathcal{R} \right] \right\} } 
        { \left\Vert \rho_j \right\Vert _{\Psi_1}^2}\\
    & = \frac{ \left( 2 + \frac{ 1 }{ K - 1 } \right) \cdot \sup_{ b \in \Lambda } \left\{ \mathrm{var} \left[ Q - \mathcal{R} \right] \right\} } 
        { \left\Vert \rho_j \right\Vert_{\Psi_1}^2 \cdot \left( n / K \right) }
    \label{eq:proof_a1_2}
  \end{align}
Based on the result that Lebesgue-$p$ norm is less or equal to the Orlicz-$\Psi_1$ norm multiplied by $p!$,\footnote{This result is pretty popular in probability and can be found in different lecture notes and books, see \citet{vanweak} as an example.}
  \begin{align}
    \mbox{eq.~(\ref{eq:proof_a1_2})}& \leqslant 4 \cdot \frac{ \left( 2 + \frac{ 1 }{ K - 1 } \right) 
          \cdot \sup_{ b \in \Lambda } \left\{ \mathrm{var} \left[ Q - \mathcal{R} \right] \right\} } { \left\Vert \rho_j \right\Vert_2^2 \cdot \left( n / K \right) } \\
    & = 4 \cdot \frac{ \left( 2 + \frac{ 1 }{ K - 1 } \right) \cdot \sup_{ b \in \Lambda } \left\{ \mathrm{var} \left[ Q - \mathcal{R} \right] \right\} } { \mathbb{E} \left[ \sup_{ b \in \Lambda } \left( \mathcal{ R }_{n_s} - \mathcal{R} \right)^2 \right] \cdot \left( n / K \right) } \\
    & \leqslant 4 \cdot \frac{ \left( 2 + \frac{ 1 }{ K - 1 } \right) \cdot \sup_{ b \in \Lambda } \left\{ \mathrm{var} \left[ Q - \mathcal{R} \right] \right\} } { \sup_{ b \in \Lambda } \left\{ \mathrm{var} \left[ Q - \mathcal{R} \right] \right\} }\\
    & \leqslant 8 + \frac{ 4 }{ K - 1 }
  \end{align}
\end{proof}
%
%
\begin{lemma}
Assume $\left\{ X_i \right\}_{i=1}^{n}$ is sampled from a stationery and mean-square ergodic process and that the autocovariance function $\mathrm{cov} \left(X_{i+l},\; X_{i} \right) := \gamma_l < \infty,\, \forall l \in \mathbf{R}$. The following inequality holds for any $\varpi \in \left[0,1\right)$,

  \begin{equation}
    \mathrm{Pr} \left( \left\vert \overline{X} - \mathbb{E} \left( X \right) \right\vert \leqslant \epsilon \right)
    \geqslant 1 - \frac{ \gamma_0 }{ \epsilon^2 n} \cdot \left( 1 + 2 V_n \left[ X \right] \right).
  \end{equation}

\noindent
where $V_n \left[ X \right] := \sum_{l=1}^{n-1} \left\vert \gamma_l\right\vert$ and $\varpi = \gamma_0 / ( n \cdot \epsilon^2 ) \cdot \left( 1 + 2 V_n \left[ X \right] \right)$.
\end{lemma}
%
%
\begin{proof}
The Chebyshev inequality shows that
  \begin{equation}
    \mathrm{Pr} \left( \vert \overline{X} - \mathbb{E} \left( X \right) \vert \leqslant \epsilon \right)
    \geqslant 1 - \frac{ \mathrm{var} \left( \overline{X} \right) }{ \epsilon^2 }.
  \end{equation}

\noindent
The variance of $\overline{X}$ over time may be expressed as

\noindent
\begin{eqnarray}
  \mathrm{var} \left( \overline{X} \right) & = & \mathbb{E} \left( \overline{X}^2 \right) - \mathbb{E} \left( \overline{X} \right)^2 \\
    & = & \frac{1}{n^2} \, \mathbb{E} \left[ \left( \sum_{i=1}^{n} X_i \right)^2 \right] - \mathbb{E} \left( X \right)^2 \\
    & = & \frac{1}{n^2} \sum_{j=1}^{n} \sum_{i=1}^{n} \mathbb{E} \left[ X_i \cdot X_j \right] - \mathbb{E} \left( X \right)^2.
\end{eqnarray}

\noindent
If we define the covariance $\gamma_{ i - j } := \mathrm{cov} \left( X_i,\; X_j \right)$ and the variance as $ \gamma_0 := \mathrm{var} \left( X \right)$, this may be expressed as

\begin{eqnarray}
  \frac{ 1 }{ n^2 } \sum_{ j = 1 }^{ n } \sum_{ i = 1 }^{ n } \mathbb{ E } \left[ X_i \cdot X_j \right] - \mathbb{ E } \left( X \right)^2
  & = & \frac{ 1 }{ n^2 } \sum_{ j = 1 }^{ n } \sum_{ i = 1 }^{ n } \gamma_{ i - j } + \mathbb{E} \left( X \right)^2 - \mathbb{E} \left( X \right)^2 \\
  & = & \frac{ 1 }{ n^2 } \sum_{ j = 1 }^{ n } \; \sum_{ l = 1 - j }^{ n - j } \gamma_{ l }\\
  & = & \frac{ \gamma_0 } { n } + \frac{ 2 }{ n^2 } \sum_{ j = 1 }^{ n - 1 }
        \sum_{ l = 1 }^{ n - j } \gamma_{ l } \\
  & \leqslant & \frac{ \gamma_0 }{ n } + \frac{ 2 }{ n^2 } \sum_{ j = 1 }^{ n - 1 }
    \sum_{ l = 1 }^{ n - j }\left\vert \gamma_{ l } \right\vert .
\end{eqnarray}

\noindent
Define $\gamma_0 \cdot V_n \left[ X \right] := \sum_{ l = 1 }^{ n - 1 } \left\vert \gamma_l \right\vert $, then

\begin{eqnarray}
  \frac{ \gamma_0 }{ n } + \frac{ 2 }{ n^2 } \sum_{ j = 1 }^{ n } \sum_{ l = 1 }^{ n - j } \left\vert \gamma_{ l } \right\vert
    & \leqslant & \frac{ \gamma_0 }{ n } + \frac{ 2 V_n \left[ X \right] \cdot \left( n - 1 \right) \cdot \gamma_0 }{ n^2 } \\
  & \leqslant & \frac{ \gamma_0 }{ n } \cdot \left( 1 + 2V_n \left[ X \right] \right).
\end{eqnarray}

\noindent
Hence the Chebyshev inequality may be generalized to a mean-square ergodic and stationery time series as
\begin{equation}
  \mathrm{Pr} \left( \left\vert \overline{X} - \mathbb{E} \left( X \right) \right\vert \leqslant \epsilon \right)
  \geqslant 1 - \frac{ \gamma_0 }{ \epsilon^2 n} \cdot \left( 1 + 2 V_n \left[ X \right] \right).
  \label{eqn:semi_cheby}
\end{equation}

\noindent
If $ V_n \left[ X \right] / n \rightarrow 0 $ and $ \gamma_0 < \infty $, the RHS of eq.~(\ref{eqn:semi_cheby}) approaches $1$ asymptotically. If we denote $ \varpi := 1 - \left( 1 + 2V_n \left[ X \right] \right) \gamma_0 / \left( \epsilon^2 n \right)$, (\ref{eqn:semi_cheby}) may be rearranged as, $\forall \varpi \in \left[ 0, 1 \right)$,

\begin{equation}
  \mathrm{Pr} \left( \vert \overline{X} - \mathbb{E} \left( X \right) \vert \leqslant \sqrt{ \frac{ \gamma_0 \cdot \left(1 + 2V_n \left[ X \right] \right) }{ n \left( 1 - \varpi \right) }} \right) \geqslant \varpi.
\end{equation}
\end{proof}
%
%
\bibliography{CVrefs}

\begin{thebibliography}{27}
\providecommand{\natexlab}[1]{#1}
\providecommand{\url}[1]{\texttt{#1}}
\expandafter\ifx\csname urlstyle\endcsname\relax
  \providecommand{\doi}[1]{doi: #1}\else
  \providecommand{\doi}{doi: \begingroup \urlstyle{rm}\Url}\fi

\bibitem[Adamczak et~al.(2008)]{adamczak2008tail}
Radoslaw Adamczak et~al.
\newblock A tail inequality for suprema of unbounded empirical processes with
  applications to markov chains.
\newblock \emph{Electronic Journal of Probability}, 13:\penalty0 1000--1034,
  2008.

\bibitem[Bartlett and Mendelson(2002)]{bartlett2002rademacher}
Peter~L Bartlett and Shahar Mendelson.
\newblock Rademacher and gaussian complexities: Risk bounds and structural
  results.
\newblock \emph{Journal of Machine Learning Research}, 3\penalty0
  (Nov):\penalty0 463--482, 2002.

\bibitem[Bengio and Grandvalet(2004)]{bengio2004no}
Yoshua Bengio and Yves Grandvalet.
\newblock No unbiased estimator of the variance of k-fold cross-validation.
\newblock \emph{Journal of machine learning research}, 5\penalty0
  (Sep):\penalty0 1089--1105, 2004.

\bibitem[Bernstein(1927)]{bernstein1927extension}
Serge Bernstein.
\newblock Sur l'extension du th{\'e}or{\`e}me limite du calcul des
  probabilit{\'e}s aux sommes de quantit{\'e}s d{\'e}pendantes.
\newblock \emph{Mathematische Annalen}, 97\penalty0 (1):\penalty0 1--59, 1927.

\bibitem[Bousquet and Elisseeff(2002)]{bousquet2002stability}
Olivier Bousquet and Andr{\'e} Elisseeff.
\newblock Stability and generalization.
\newblock \emph{Journal of Machine Learning Research}, 2\penalty0
  (Mar):\penalty0 499--526, 2002.

\bibitem[Elliott(1982)]{elliott1982stochastic}
Robert~James Elliott.
\newblock \emph{Stochastic calculus and applications}, volume~18.
\newblock Springer-Verlag, 1982.

\bibitem[Kim(2009)]{kim2009estimating}
Ji-Hyun Kim.
\newblock Estimating classification error rate: Repeated cross-validation,
  repeated hold-out and bootstrap.
\newblock \emph{Computational Statistics \& Data Analysis}, 53\penalty0
  (11):\penalty0 3735--3745, 2009.

\bibitem[Kohavi(1995)]{kohavi1995study}
Ron Kohavi.
\newblock A study of cross-validation and bootstrap for accuracy estimation and
  model selection.
\newblock In \emph{Proceedings of the 14th international joint conference on
  Artificial intelligence}, volume~2, pages 1137--1145, 1995.

\bibitem[Lecue(2009)]{Lecue09tool}
Guillaume Lecue.
\newblock Basic tools from empirical processes theory applied to the compress
  sensing problem.
\newblock \url{http://perso-math.univ-mlv.fr/users/
  banach/Fallschool2009/notes/
  L2009BasicToolsEmpiricalProcessesTheoryAppliedCompressSensing.
  concentration.pdf}, 2009.

\bibitem[Lederer et~al.(2014)Lederer, Van De~Geer, et~al.]{lederer2014new}
Johannes Lederer, Sara Van De~Geer, et~al.
\newblock New concentration inequalities for suprema of empirical processes.
\newblock \emph{Bernoulli}, 20\penalty0 (4):\penalty0 2020--2038, 2014.

\bibitem[Lim and Yu(2016)]{lim2016estimation}
Chinghway Lim and Bin Yu.
\newblock Estimation stability with cross-validation (escv).
\newblock \emph{Journal of Computational and Graphical Statistics}, 25\penalty0
  (2):\penalty0 464--492, 2016.

\bibitem[Massart(2000)]{massart2000constants}
Pascal Massart.
\newblock About the constants in talagrand's concentration inequalities for
  empirical processes.
\newblock \emph{Annals of Probability}, pages 863--884, 2000.

\bibitem[McDonald et~al.(2011)McDonald, Shalizi, and Schervish]{shalizi2011}
Daniel~J McDonald, Cosma~Rohilla Shalizi, and Mark Schervish.
\newblock Generalization error bounds for stationary autoregressive models.
\newblock \emph{arXiv e-print 1103.0942}, 2011.

\bibitem[Mohri and Rostamizadeh(2009)]{mohri2009rademacher}
Mehryar Mohri and Afshin Rostamizadeh.
\newblock Rademacher complexity bounds for non-i.i.d.\ processes.
\newblock In Daphne Koller, Dale Schuurmans, Yoshua Bengio, and Leon Bottou,
  editors, \emph{Advances in Neural Information Processing Systems 21}, pages
  1097--1104. Curran Associates, Inc., 2009.

\bibitem[Mohri and Rostamizadeh(2010)]{mohri2010stability}
Mehryar Mohri and Afshin Rostamizadeh.
\newblock Stability bounds for stationary $\varphi$-mixing and $\beta$-mixing
  processes.
\newblock \emph{Journal of Machine Learning Research}, 11\penalty0
  (Feb):\penalty0 789--814, 2010.

\bibitem[Peng(2004)]{peng2004filtration}
Shige Peng.
\newblock Filtration consistent nonlinear expectations and evaluations of
  contingent claims.
\newblock \emph{Acta Mathematicae Applicatae Sinica (English Series)},
  20\penalty0 (2):\penalty0 191--214, 2004.

\bibitem[Stone(1974)]{stone74}
M.~Stone.
\newblock Cross-validatory choice and assessment of statistical predictions.
\newblock \emph{Journal of the Royal Statistical Society, Series B
  (Methodological)}, 36\penalty0 (2):\penalty0 111--147, 1974.

\bibitem[Stone(1977)]{stone77}
M.~Stone.
\newblock An asymptotic equivalence of choice of model by cross-validation and
  {Akaike}'s criterion.
\newblock \emph{Journal of the Royal Statistical Society, Series B
  (Methodological)}, 39\penalty0 (1):\penalty0 44--47, 1977.

\bibitem[Striebel(1975)]{striebel2013optimal}
Charlotte Striebel.
\newblock \emph{Optimal control of discrete time stochastic systems}, volume
  110.
\newblock Springer, 1975.

\bibitem[Talagrand(1994)]{talagrand1994supremum}
Michel Talagrand.
\newblock The supremum of some canonical processes.
\newblock \emph{American Journal of Mathematics}, 116\penalty0 (2):\penalty0
  283--325, 1994.

\bibitem[Tao(2009)]{taoconstant}
Terence Tao.
\newblock Talagrand's concentration inequalities.
\newblock
  \url{https://terrytao.wordpress.com/2009/06/09/talagrands-concentration-inequality/},
  2009.

\bibitem[Valiant(1984)]{valiant1984theory}
Leslie~G Valiant.
\newblock A theory of the learnable.
\newblock \emph{Communications of the ACM}, 27\penalty0 (11):\penalty0
  1134--1142, 1984.

\bibitem[Van~der Vaart and Wellner(1996)]{vanweak}
AW~Van~der Vaart and JA~Wellner.
\newblock \emph{Weak Convergence and Empirical Processes}.
\newblock Springer, New York, 1996.

\bibitem[Vapnik(1998)]{vapnik1998statistical}
Vladimir~Naumovich Vapnik.
\newblock \emph{Statistical learning theory}, volume~1.
\newblock Wiley New York, 1998.

\bibitem[Xu and Liang(2001)]{xu2001monte}
Qing-Song Xu and Yi-Zeng Liang.
\newblock Monte carlo cross validation.
\newblock \emph{Chemometrics and Intelligent Laboratory Systems}, 56\penalty0
  (1):\penalty0 1--11, 2001.

\bibitem[Yan(1985)]{jayan1985commutability}
Jia-An Yan.
\newblock On the commutability of essential infimum and conditional expectation
  operations.
\newblock \emph{Science Bulletin}, 8:\penalty0 004, 1985.

\bibitem[Yu(1994)]{yu1994rates}
Bin Yu.
\newblock Rates of convergence for empirical processes of stationary mixing
  sequences.
\newblock \emph{The Annals of Probability}, pages 94--116, 1994.

\end{thebibliography}

\end{document}